\documentclass[lettersize,journal]{IEEEtran}
\usepackage{array}
\usepackage{textcomp}
\usepackage{stfloats}
\usepackage{url}
\usepackage{verbatim}
\usepackage[pagebackref,breaklinks,colorlinks]{hyperref}
\usepackage[space]{cite}
\usepackage[english]{babel}

\usepackage{graphicx}
\usepackage{color}
\usepackage{tabularx} 
\usepackage{threeparttablex} 
\usepackage{multirow}
\usepackage{multirow}
\usepackage{makecell}
\usepackage[table,xcdraw]{xcolor}
\usepackage[normalem]{ulem}
\useunder{\uline}{\ul}{}

\usepackage{amsthm}
\usepackage{amsmath,amsfonts}
\usepackage{algorithmic}
\usepackage[ruled,vlined,linesnumbered]{algorithm2e}
\usepackage{amssymb}

\newtheorem{theorem}{Theorem}

\newtheorem{assumption}[theorem]{Assumption}

\newtheorem{proposition}[theorem]{Proposition}

\newcommand{\revised}[1]{{#1}}
\usepackage{subfigure}

\begin{document}

\title{G3Reg: Pyramid Graph-based Global Registration using Gaussian Ellipsoid Model}
%
\author{Zhijian Qiao, Zehuan Yu, Binqian Jiang, Huan Yin, and Shaojie Shen
\thanks{This work was supported in part by the Hong Kong Center for Construction Robotics (InnoHK center supported by Hong Kong ITC), in part by the HKUST Postgraduate Studentship, and in part by the HKUST-DJI Joint Innovation Laboratory.}
\thanks{The authors are with the Department of Electronic and Computer Engineering, The Hong Kong University of Science and Technology, Hong Kong, China. E-mail: zqiaoac@connect.ust.hk, zyuay@connect.ust.hk, bjiangah@connect.ust.hk, eehyin@ust.hk, eeshaojie@ust.hk.}
\thanks{Corresponding author: Huan Yin}
}



\maketitle

\begin{abstract}
This study introduces a novel framework, G3Reg, for fast and robust global registration of LiDAR point clouds. In contrast to conventional complex keypoints and descriptors, we extract fundamental geometric primitives, including planes, clusters, and lines (PCL) from the raw point cloud to obtain low-level semantic segments. Each segment is represented as a unified Gaussian Ellipsoid Model (GEM), using a probability ellipsoid to ensure the ground truth centers are encompassed with a certain degree of probability. Utilizing these GEMs, we present a distrust-and-verify scheme based on a Pyramid Compatibility Graph for Global Registration (PAGOR). Specifically, we establish an upper bound, which can be traversed based on the confidence level for compatibility testing to construct the pyramid graph. Then, we solve multiple maximum cliques (MAC) for each level of the pyramid graph, thus generating the corresponding transformation candidates. In the verification phase, we adopt a precise and efficient metric for point cloud alignment quality, founded on geometric primitives, to identify the optimal candidate. The algorithm's performance is validated on three publicly available datasets and a self-collected multi-session dataset. Parameter settings remained unchanged during the experiment evaluations. The results exhibit superior robustness and real-time performance of the G3Reg framework compared to state-of-the-art methods. Furthermore, we demonstrate the potential for integrating individual GEM and PAGOR components into other registration frameworks to enhance their efficacy.
\end{abstract}

\def\abstractname{Note to Practitioners}
\begin{abstract}
Our proposed method aims to perform global registration for outdoor LiDAR point clouds. Our methodology, which extracts point cloud segments and utilizes their centers for registration, differs from conventional approaches that rely on keypoints and descriptors. We further propose GEM to model the uncertainty of the centers and embed it into our distrust-and-verify framework. In theory, our method can be applied to any registration task that involves primitives representable as sets of Gaussians or points. Additionally, practitioners should consider the following to enhance applicability. First, practitioners can fine-tune the parameters of the segmentation algorithm to generate more repeatable segmentation results. Second, although our default setting uses four compatibility test thresholds, fewer may suffice, especially when translations between point clouds are minor. Finally, for geometrically uninformative segments such as vegetation, consider extracting descriptors\cite{yin2023segregator} within these segments to increase correspondences.
\end{abstract}


\section{Introduction}
\label{sec:introduction}

\begin{figure}[t]
	\centering
	\includegraphics[width=\linewidth]{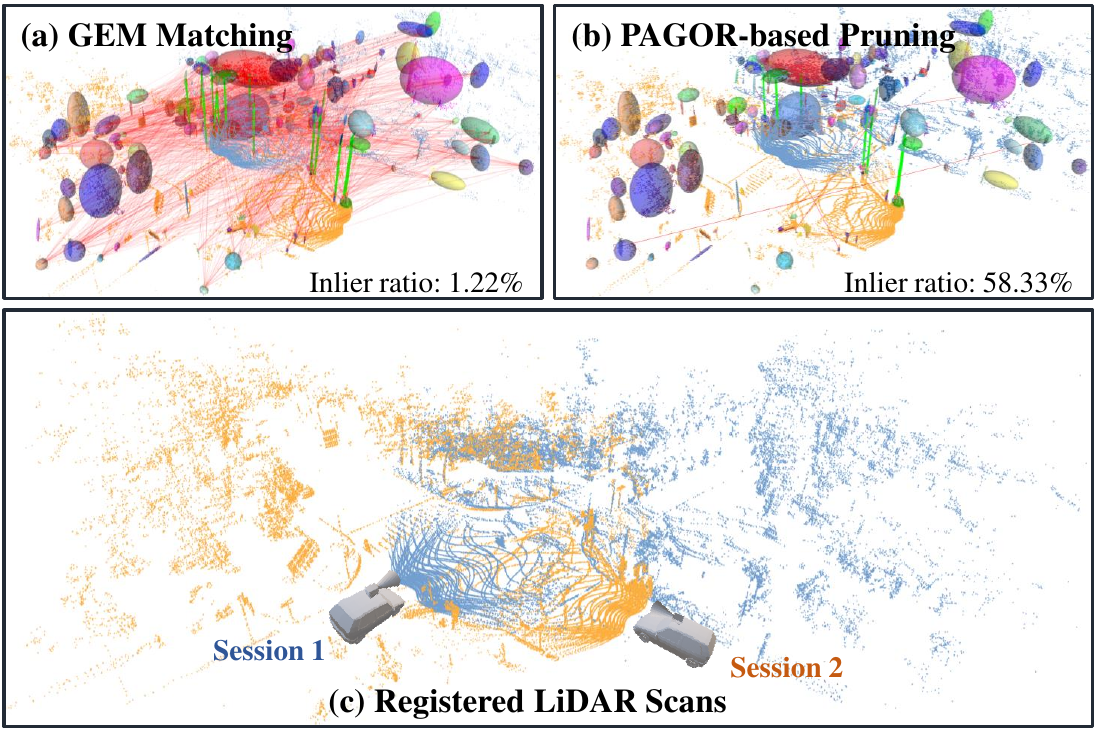}
	\caption{Global registration procedure using our proposed G3Reg. (a) Front-end GEM matching for Gaussian ellipsoid modeling and correspondence building. (b) Back-end PAGOR for outlier pruning and pose estimation under a distrust-and-verify scheme. (c) A challenging global registration result at a road intersection, in which the two LiDAR point clouds are with low overlap and a large view difference.}
	\label{figure:overview}
    \vspace{-0.4cm}
\end{figure}

\IEEEPARstart{A}{ligning} two individual point clouds is a fundamental and indispensable operation in robotics and autonomous systems. Global registration, which estimates the transformation without prior information, has become crucial for tasks such as \revised{loop closure\cite{liu2019lpd,pramatarov2022boxgraph,dube2017segmatch}, multi-session SLAM\cite{yin2023automerge,yu2023multi} and relocalization\cite{yin2023survey,xu2023ring++}}. The standard pipeline of global registration can typically be divided into front-end and back-end stages, which are responsible for establishing putative correspondences $\mathcal{I}$ and estimating the transformation $(\mathbf{R}, \mathbf{t})$, respectively. The transformation estimation can be formulated as a least squares problem:

\begin{equation}
\label{equ:reg_ls}
\min _{\mathbf{R} \in \operatorname{SO}(3), \mathbf{t} \in \mathbb{R}^3} \sum_{\left(x_k, y_k\right) \in \mathcal{I}} \rho\left(r\left(y_k, f\left(x_k \mid \mathbf{R}, \mathbf{t}\right)\right)\right)
\end{equation}
where $\rho(\cdot)$ represents the robust kernel function, $r(\cdot)$ is the residual function, and $f(\cdot)$ denotes the rigid transformation.

\revised{
Conventional approaches for establishing $\mathcal{I}$ rely on complex methods of keypoint and descriptor extraction. These methods include handcrafted techniques\cite{rusu2009fast} and deep learning-based local descriptors\cite{yew2020rpm,ao2021spinnet,poiesi2022learning}. However, these approaches encounter challenges such as excessive processing time and reduced descriptor consistency, primarily due to varying viewpoints. Specifically, during the construction and matching of descriptors, frequent nearest neighbor searches are performed in both the spatial and descriptor space, which can be time-consuming. Additionally, changes in viewpoint result in occlusions and density variations, resulting in alterations in the neighborhood of each point. This further affects the computation of local features, such as normals\cite{rusu2009fast,yew2020rpm} and local reference frames\cite{ao2021spinnet,poiesi2022learning}, which are essential for the consistency of descriptors relying on these local features. Consequently, an extremely high outlier ratio (e.g., 99\%) is observed, significantly undermining the quality of the registration process.
}


To address this issue, recent advances\cite{yang2020teaser, lusk2021clipper, shi2021robin, lim2022single, yin2023segregator, zhang20233d} suggest utilizing maximal clique inlier selection and coupling it with a robust estimator for transformation estimation. The inlier clique has been proven to belong to at least one maximal clique in the compatibility graph\cite{yang2020teaser}. To efficiently find this maximal clique, the maximum clique with the largest cardinality is computed. However, this strategy does not always work due to inappropriate threshold selection in the compatibility test\cite{shi2021robin}, especially under extreme viewpoint changes and repetitive patterns, as discussed in\cite{qiaoiros2023}.

In this study, we present a novel framework, named G3Reg, for the global registration of LiDAR point clouds that effectively addresses these issues. Our first key insight is to use segments and classify them into geometric primitives (planes, clusters, and lines) as a replacement for complex keypoints and descriptors. This approach has two advantages: firstly, using segments directly avoids the loss of geometric information caused by traditional keypoints; secondly, it eliminates the need for complex keypoint and descriptor extraction, thus reducing computational time. Our second key insight is a novel distrust-and-verify scheme that generates multiple transformation candidates and leverages compressed raw point clouds to verify and select the most appropriate candidate. Specifically, multiple maximum cliques are obtained by constructing a pyramid compatibility graph using a traversable confidence-related upper bound for the compatibility test. Each maximum clique is then used to estimate a transformation candidate.

This work is an extended version of our earlier conference paper\cite{qiaoiros2023}. In comparison, the journal version presents several notable enhancements, including a more efficient front-end approach that eliminates the need for the semantic segmentation network. Additionally, we have significantly improved our distrust-and-verify framework in terms of both efficiency and robustness. The contributions of this work are outlined as follows:

\begin{itemize}
\item We introduce a novel segment-based front-end approach to obtain putative correspondences. In this approach, each segment is parameterized by the proposed GEM, which lays the groundwork for the following multi-threshold compatibility test and enables distribution-to-distribution registration at the back end.
\item We propose a distrust-and-verify scheme that generates multiple transformation candidates by distrusting the compatibility test and verifying their point cloud alignment quality. To achieve this, a pyramid compatibility graph is constructed using the multi-threshold compatibility test, and a graduated MAC solver is proposed to solve the maximum clique on each level. Finally, transformation candidates are estimated from these cliques, and an evaluation function is designed to select the optimal one.
\item We conduct extensive evaluations of our proposed G3Reg on three publicly available datasets and a self-collected multi-session dataset in the real world. The experimental results demonstrate the superiority of our proposed method in terms of robustness and real-time performance compared to state-of-the-art handcrafted and deep learning methods.
\item 
To enhance comprehension and encourage the adoption of our framework, we have released our source code~\footnote{\href{https://github.com/HKUST-Aerial-Robotics/G3Reg}{https://github.com/HKUST-Aerial-Robotics/G3Reg}} as open-source to the community. Our framework not only facilitates the replication of experiments but also offers the flexibility to modify individual components like detectors, matchers, maximum clique solvers, transformation estimators, and scalability to suit different datasets.
\end{itemize}
\section{Related Work}
\label{sec:related}

Global point cloud registration can be categorized into two primary models: correspondence-free and correspondence-based. The correspondence-free model achieves global registration using techniques like Fourier analysis\cite{bernreiter2021phaser}, branch and bound strategies\cite{yang2015go}, and learned feature alignment techniques\cite{li2021pointnetlk, aoki2019pointnetlk}. However, correspondence-free models are generally time-consuming for global transformation search. This paper proposes a correspondence-based method, which involves various modules, including discriminative keypoint detection, descriptor extraction, efficient correspondence ranking for outlier pruning, and robust transformation estimation.  In this section, we delve into related works that concentrate on these individual modules.

\subsection{3D Point Matching}

\subsubsection{3D Keypoint Detection}

The process of 3D point matching involves two main steps: keypoint detection and descriptor designing. In the case of LiDAR point clouds in large-scale scenes, keypoint detection is crucial for reducing the number of point clouds that require registration, thus alleviating the burden on the registration algorithm.

Manual keypoint detection methods use local geometric properties to assess the saliency of a point. For example, ISS\cite{zhong2009intrinsic} uses the eigenvalue ratio of the neighborhood covariance matrix to evaluate point prominence. Similarly, KPQ\cite{mian2010repeatability} utilizes the largest eigenvalue ratio to identify candidate keypoints and calculates curvature for prominence estimation. On the other hand, the deep learning method D3Feat\cite{bai2020d3feat} predicts keypoint scores for each point in dense point clouds, assigning higher scores to matchable matches.

In practical applications, LiDAR-based keypoint detection methods are often affected by factors like noise, occlusion, and point cloud density, resulting in poor repeatability. Additionally, keypoints may discard important geometric information, such as cluttered vegetation and smooth walls. Therefore, methods like Segmatch\cite{dube2017segmatch} and BoxGraph\cite{pramatarov2022boxgraph} have been developed, which extract segments to represent point clouds and use their centers for registration. Segment-based registration methods retain the structural information of the entire point cloud, showing robustness to noise and density variations. \revised{The common parameterization of segments could be 3D bounding box\cite{pramatarov2022boxgraph} or the ellipsoid\cite{liu2021volumon,nicholson2018quadricslam}.} However, the repeatability of their center estimation is significantly influenced by viewpoints.

In our proposed method, we also utilize segments for global registration. The key difference lies in our introduction of a plane-assisted segmentation method to achieve segments with higher repeatability. Additionally, we employ GEM to model the uncertainty of segment centers, ensuring probabilistic coverage of the ground truth center. This approach enhances the robustness and accuracy of our registration process.

\subsubsection{3D Descriptor-based Matching}
\label{sec:rw:des}
\revised{
SOTA works focus on constructing descriptors that are rigid transformation invariant, discriminative, and highly generalizable. To achieve invariance, FPFH\cite{rusu2009fast} and RPMNet\cite{yew2020rpm} compute features such as distances and angles between points and normals that are invariant to transformations to construct their descriptors. Similarly, the local reference frame (LRF) is utilized in techniques like SHOT\cite{tombari2010unique}, SpinNet\cite{ao2021spinnet} and GeDi\cite{poiesi2022learning} to maintain the invariance to the rotation and translation. Furthermore, RoReg\cite{wang2023roreg} introduces the icosahedral group and rotation-equivariant layers to achieve rotation equivariance, leading to rotation invariance through average pooling.}

\revised{However, as the viewpoint changes, descriptor consistency can be compromised by occlusions and density variations, which means the descriptors of the same point from two different point clouds may not be similar. Furthermore, these factors can affect the local feature estimation, such as normals and LRFs, which in turn may degrade the descriptors' discriminative and generalization capabilities.}

\revised{In contrast to point-based methods, our proposed GEM use segments for registration, which naturally reduces the impact of density changes since these do not significantly alter the segment's shape. Additionally, our method considers the uncertainty of segment center through the proposed probability coverage of the center (see Assumption \ref{assump:center}). While GEM does not rely on strong descriptors, we use a weak descriptor coupled with a Mutual-K-Nearest Neighbors (MKNN) matching strategy to establish the initial correspondence set $\mathcal{I}$ in Problem~\eqref{equ:reg_ls}. This approach effectively prevents degeneracy, where inlier correspondences are fewer than three, thus ensuring a more robust registration.}

\subsection{Outlier Pruning}

The robust loss function $\rho(\cdot)$ in Problem~\eqref{equ:reg_ls} is known for effectively rejecting outliers, but it may face challenges when dealing with problems involving high outlier ratios, such as up to 99\% outliers. To reduce the outlier ratio to a level where existing solvers (such as RANSAC and GNC\cite{yang2020graduated}) perform well, a common approach in the computer vision community involves employing reciprocity checks and ratio tests\cite{lowe2004distinctive}. Unlike the aforementioned approaches that heavily rely on descriptor performance, GORE\cite{bustos2017guaranteed} reformulates the problem, leveraging deterministic geometric properties to establish an upper bound on the inlier set size subproblem and a lower bound on the optimal solution, effectively eliminating incorrect matches.

In recent years, graph-theoretic methods have emerged as powerful tools in various registration works, demonstrating exceptional performance regardless of descriptor quality. Based on the compatibility graph, Enqvist \textit{et al.}\cite{enqvist2009optimal} propose using vertex cover to identify mutually consistent correspondences. MV (Mutual Voting)\cite{yang2023mutual} introduces a mutual voting method for ranking 3D correspondences, enabling vertices and edges to refine each other in a mutually beneficial manner. Another widely used and stricter algorithm is the maximal clique inlier selection\cite{yang2020teaser, bailey2000data, lusk2021clipper, lusk2022graffmatch, yin2023segregator,zhang20233d}, which requires every correspondence to be mutually compatible. However, finding the maximal clique is known to be an NP-hard problem, and its runtime often becomes impractical as the number of correspondences increases. To address this challenge, ROBIN\cite{shi2021robin} proposes a faster alternative by computing the maximum $k$-core of the graph.

In our proposed method, we also adopt the maximum clique approach but introduce a traversable confidence-related threshold to construct a pyramid compatibility graph and solve multiple maximum cliques in a graduated manner. This approach effectively increases the probability that inliers make up one of the solved maximum cliques.

\subsection{Robust Transformation Estimation}

Robust transformation estimation has been extensively studied in various research fields\cite{carlone2022estimation}. It can be formulated using different approaches, including least trimmed squares (LTS) in classical robust statistics\cite{chetverikov2002trimmed}, consensus maximization (MC) in computer vision\cite{peng2022arcs}, quadratic pose estimation problems\cite{wu2022quadratic}, and truncated least squares (TLS) in robotics\cite{yang2020teaser, zhou2016fast}. Trimmed ICP\cite{chetverikov2002trimmed} selects potential inlier correspondences and estimates the transformation using a predefined trimming percentage, which is often unknown in real applications. On the other hand, MC defines a consensus set that contains consistent correspondences with residuals below a threshold and aims to find the optimal $(\mathbf{R}, \mathbf{t})$ that maximizes the size of the consensus set. However, MC becomes NP-hard\cite{chin2018robust} when $\mathbf{R} \in \operatorname{SO}(3)$, leading to the adoption of the suboptimal but efficient RANSAC algorithm to find the outlier-free set. Nevertheless, the runtime of RANSAC increases exponentially with the outlier rate.

In this study, we formulate the registration problem as a truncated least squares (TLS) problem. To address this problem, we leverage the Black-Rangarajan duality\cite{black1996unification} to rewrite TLS using the corresponding outlier process. To solve the TLS problem, we adopt the fast heuristic Graduated Non-Convexity (GNC) algorithm based on alternating optimization\cite{yang2020graduated}. The GNC approach starts with a convex approximation of the cost function and gradually refines it to recover the original function. By employing graph-theoretic outlier pruning, GNC performs well in scenarios with moderate outlier rates, such as 80\%. The primary difference between our formulation and conventional ones\cite{yang2020teaser} lies in the usage of a distribution-to-distribution residual rather than a point-to-point residual, following a similar approach to Generalized ICP (GICP)\cite{segal2009generalized}.
\begin{figure*}[htbp]
	\centering
	\includegraphics[width=\linewidth]{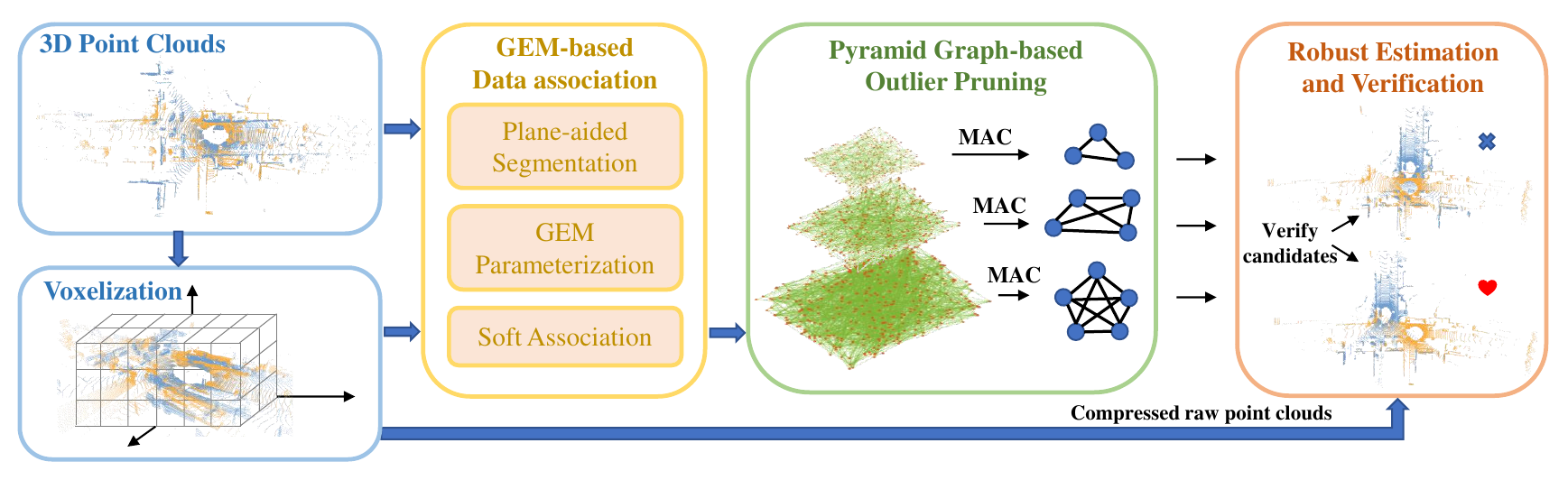}
	\caption{The proposed global registration method G3Reg adopts a distrust-and-verify framework, consisting of three components. First, putative correspondences between input point clouds are established via GEM-based data association (block in yellow). Then a pyramid compatibility graph is constructed to generate maximal cliques to obtain transformation candidates (block in green). Finally, the original point cloud information is re-utilized to select the optimal candidate (block in brownish-orange). }
	\label{figure:pipeline}
\end{figure*}
\section{Problem Formulation}

Given two 3D LiDAR point clouds, denoted as $\mathcal{X}$ and $\mathcal{Y}$, our primary goal is to estimate the optimal rigid transformation, represented as $(\tilde{\mathbf{R}}, \tilde{\mathbf{t}})$, that best aligns the source point cloud $\mathcal{X}$ with the target point cloud $\mathcal{Y}$. To achieve this, we first extract Gaussian Ellipsoid Models (GEM) from the segments within the point clouds during the front-end processing phase (Section~\ref{sec:gmm_front}). Consequently, we can denote $\mathcal{X}$ and $\mathcal{Y}$ as a series of GEMs, $\{x_1 \ldots x_{N_x} \}$ and $\{y_1 \ldots y_{N_y} \}$ respectively.

In the subsequent back-end phase of the framework, our strategy for estimating the optimal transformation parameters $(\tilde{\mathbf{R}}, \tilde{\mathbf{t}})$ employs a distrust-and-verify scheme. This scheme generates multiple transformation candidates using the extracted GEMs and then selects the most probable candidate that optimally aligns $\mathcal{X}$ and $\mathcal{Y}$.

To hypothesize a candidate transformation $(\mathbf{R}^*, \mathbf{t}^*)$, we construct a compatibility graph (Section \ref{sec:graph_build}) for the original putative correspondence set $\mathcal{I}$ and derive a potential inlier set $\mathcal{I}^*$ by identifying the maximum clique within the graph (Section \ref{sec:mac_solver}). We adapt the formulation presented in Equation \eqref{equ:reg_ls} to express the registration problem as follows:

\begin{equation}
\label{equ:sub_reg}
\mathbf{R}^*, \mathbf{t}^*=\min \sum_{\left(x_k, y_k\right) \in \mathcal{I}^*} \rho\left(r\left(y_k, f\left(x_k \mid \mathbf{R}, \mathbf{t}\right)\right)\right)
\end{equation}

In this equation, $\mathcal{I}$ is replaced with $\mathcal{I}^*$ compared to the formulation in Equation \eqref{equ:reg_ls}. The robust loss function $\rho(\cdot)$ and the residual function $r(\cdot)$ of two associated GEMs are defined in Section \ref{sec:tf_solve}.

By traversing the compatibility test threshold across various confidence levels (Section \ref{sec:comp_test}) to construct a pyramid compatibility graph, we can derive multiple potential $\mathcal{I}^{*}_{m}$ from their maximum cliques. We can then solve the corresponding hypothetical transformation $(\mathbf{R}^{*}_{m}, \mathbf{t}^{*}_{m})$ using Equation \eqref{equ:sub_reg}. Subsequently, we introduce an evaluation function $g$ (Section~\ref{sec:verify}) to determine the most suitable transformation based on the geometric information of $\mathcal{X}$ and $\mathcal{Y}$, as follows:

\begin{equation}
\label{eq:eval}
\tilde{\mathbf{R}}, \tilde{\mathbf{t}}=\underset{\tilde{\mathbf{R}} \in \left\{\mathbf{R}_{m}^{*}\right\}, \tilde{\mathbf{t}} \in \left\{\mathbf{t}_{m}^{*}\right\}}{\arg \min } g\left(\mathbf{R}_{m}^{*}, \mathbf{t}_{m}^{*} \mid \mathcal{X}, \mathcal{Y} \right)
\end{equation}

In summary, our proposed G3Reg method firstly introduces the novel concept of GEM and utilizes it to obtain the putative correspondence set. It then follows a carefully designed distrust-and-verify framework that facilitates fast and robust registration.

\section{GEM-based Data Association}
\label{sec:gmm_front}

In this section, we present our method for extracting Gaussian Ellipsoid Models (GEMs) from two given point clouds and acquiring the initial putative correspondences, $\mathcal{I}$. This entire procedure is composed of three steps. First, we design a plane-aided segmentation algorithm to process the input LiDAR scans, resulting in a series of plane, line and cluster (PCL) segments (Section~\ref{sec:pas}). Second, for each of these segments, we model it as a GEM. This is achieved by estimating its statistical Gaussian distribution parameters along with an additional pseudo-Gaussian parameter (Section~\ref{sec:para_est}). Lastly, based on the GEM parameters, we select a subset of these GEMs that are deemed significant and implement a soft association strategy. This strategy is facilitated through the use of mutual K nearest neighbors (Section~\ref{sec:soft_ass}).

\subsection{Plane-aided Segmentation}
\label{sec:pas}
\begin{figure}[htbp]
	\centering
	\includegraphics[width=0.98\linewidth]{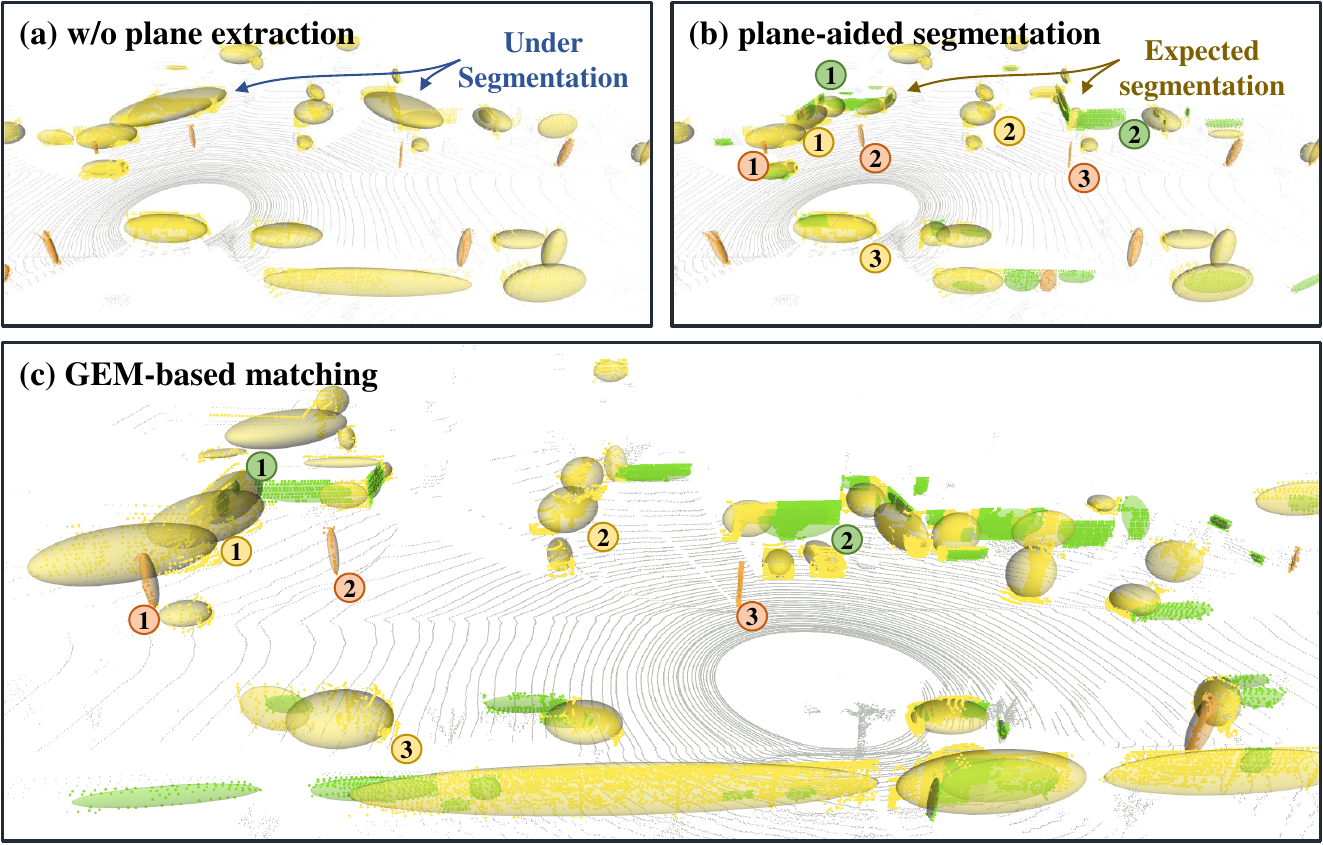}
	\caption{Visualization of GEM-based data association. Green, yellow, and orange represent planes, clusters, and lines respectively. (a) and (b) show the difference in segmentation results without or with plane assistance. As the arrow indicates, the former leads to under-segmentation. (c) is a different frame from (b). We use different colors and indices to illustrate the matching between GEMs of these two frames.}
	\label{figure:plane_aided}
\end{figure}

The initiation of GEM extraction involves identifying repeatable segments within a LiDAR point cloud from varying viewpoints. We segment point clouds with reference to TRAVEL\cite{oh2022travel}; however, a significant distinction lies in our categorization of the resulting segments into three types: planes, lines, and clusters (PCL). This methodology is underpinned by three considerations. First, planes and lines, being distinct geometric structures, are advantageous for precise point cloud registration. Second, plane segmentation helps prevent extensive portions of the point cloud segments from being incorrectly merged, resulting in under-segmentation. This strengthens the repeatability of the derived clusters. Lastly, these low-level semantics assist in data association. Figure~\ref{figure:plane_aided} presents a visual comparison of outcomes obtained with and without plane-assisted segmentation, and also the GEM-based matching proposed in our approach.

More specifically, we first voxelise the point cloud with a voxel size of $s_v$, and determine whether each voxel forms a plane based on the ratio of the smallest eigenvalue $\lambda_3$ to the second smallest eigenvalue $\lambda_2$ and its threshold $\sigma_p$. We then apply region growing to merge plane voxels into plane segments. The decision to merge plane voxels is predicated on several factors, including point-to-plane distance, the similarity of normal vectors $n_1$ and $n_2$, and adjacency relations. These conditions are elaborated as follows:

\begin{enumerate}
\item Distance condition: $abs(n_{1}^{\mathrm{T}}(p_1-p_2))$ and $abs(n_{2}^{\mathrm{T}}(p_1-p_2))$ are smaller than a threshold $\sigma_{p2n}$. Here, $p$ represents the center of plane points.
\item Normal condition: $abs(n_{1}^{\mathrm{T}}n_2)$ is smaller than a threshold $\sigma_n$.
\item Adjacency condition: two planes should contain at least one voxel that is adjacent to each other.
\end{enumerate}

\revised{Subsequently, we segment the remaining point cloud into clusters using the TRAVEL algorithm~\cite{oh2022travel}. This step is the most time-consuming in the GEM extraction process, with a time complexity of \(O(n\log(n))\) where $n$ is the total number of nodes in a range image row.} Furthermore, for each cluster, we employ RANSAC to fit a line, employing a point-to-line distance threshold $\sigma_{p2l}$. Clusters that exhibit an inlier rate exceeding a predefined threshold $\sigma_{line}$ are identified as lines. Ultimately, every voxel is re-labeled based on the segmented PCL segments and is used in the verification stage for evaluating alignment quality (Section~\ref{sec:verify}).

\subsection{Parameter Estimation}
\label{sec:para_est}

\begin{algorithm}[htbp]
    \label{alg:gem}
    \SetAlgoLined
    \NoCaptionOfAlgo
    \caption{\textbf{Data Structure 1}: Gaussian Ellipsoid Model}
    \SetKwProg{GEMNode}{Struct}{:}{end}
    \GEMNode{$\mathtt{GEM}$}{
        PCL $\mathtt{type}$; \% Plane, cluster or line\\
        Vector3 $\mathtt{\mu}$; \% Statistical center\\
        Matrix3 $\Sigma$; \% Statistical covariance\\
        Matrix3 $\hat{\Sigma}$; \% Pseudo covariance\\
        Vector3 $\hat{\lambda}$; \% Eigenvalues of $\hat{\Sigma}$\\
    }
\end{algorithm}   

In this section, we represent the low-level semantic segments obtained in the previous step via a standardized parameterization mentioned as the Gaussian Ellipsoid Model (GEM). The attributes of the GEM are outlined in \textbf{Data Structure~\ref{alg:gem}}. Each GEM's type is associated with its specific geometric primitive. We directly derive the statistical Gaussian distribution parameters from the parameters of the contained voxels, thereby eliminating the need for recomputation:
\begin{equation}
\begin{aligned}
& \mu=\frac{1}{\sum_{k=1}^{N_{v}} N_k} \sum_{k=1}^{N_{v}} N_k \mu_{k} \\
& \Sigma=\frac{1}{\sum_{k=1}^{N_{v}} N_k} \sum_{k=1}^{N_{v}}\left(N_k\left(\Sigma_k+\mu_{k} \mu_{k}^{\mathrm{T}}\right)\right)-u u^{\mathrm{T}}
\end{aligned}
\end{equation}
where $N_{v}$ is the number of contained voxels, and ${N_{v}}$ and $(\mu_{k},\Sigma_k)$ denote the point number, statistical center and covariance matrix of the $k$-th voxel, respectively.

Generally, the distribution of an object in the real world adheres to a uniform distribution. Therefore, the use of statistical Gaussian covariance as a descriptor may result in distortions and limitations, particularly when confronted with variations in point density. To address this issue, we introduce pseudo-Gaussian parameters to model the uncertainty associated with a segment's center point $\mu$, based on Assumption \ref{assump:center}, as illustrated in Figure \ref{fig:gem_obb}.

\begin{assumption}[Probability Coverage of Center] \label{assump:center}
Given an observed segment, the observation's incompleteness implies that the statistically derived $\mu$ cannot represent its true center $\hat{\mu}$. However, we expect $\hat{\mu}$ to converge within the minimal 3D Oriented Bounding Box (OBB) that encloses the segment, with a high degree of probability. Moreover, the difference $\delta\mu$ between $\mu$ and $\hat{\mu}$ is constrained within this 3D OBB, as both $\mu$ and $\hat{\mu}$ converge within it.
\end{assumption}

Following this assumption, we estimate pseudo-Gaussian parameters in two stages. Initially, we regress the 3D OBB by projecting the points onto a 2D plane along a specified axis (normal for plane and cluster, direction for line). We then derive the 2D OBB from the projected points as per \cite{freeman1975determining}. The 3D OBB is finally attained by merging the 2D OBB with its projection axis.

In the subsequent stage, we derive the 3D Oriented Bounding Ellipsoid (OBE) that is tangential to the 3D OBB and designate it as the probability ellipsoid for the center point $\mu$ of the GEM. Assuming that $\mu \in \mathcal{N}(\hat{\mu}, \hat{\Sigma})$, this probability ellipsoid can be represented as follows:

\begin{equation}
(\mu-\hat{\mu})^{\mathrm{T}} \hat{\Sigma}^{-1}(\mu-\hat{\mu}) \leq \chi_{(p)}^2
\end{equation}
where $\chi_{(p)}^2$ is the 3-DoF $\chi^2$ value with a probability $(1-p)$ to reject the possible center outside of the ellipsoid. Based on the estimated size $s$ and orientation $\mathbf{R}_{o}$ of OBE, we obtain
\begin{equation}
\hat{\lambda}i=\left(\frac{s_i}{2 \sqrt{\chi_{(p)}^2}}\right)^2
\end{equation}
\begin{equation}
\hat{\Sigma}=\mathbf{R}_{o} \operatorname{diag}(\hat{\lambda}) {\mathbf{R}_{o}}^{\mathrm{T}}
\end{equation}

We set $\chi_{(p)}^2=\chi_{(0.05)}^2=7.815$ to ensure a 95\% probability that the ground truth center converges. The derived pseudogaussian parameters offer an account of the uncertainty in the segment's center and will be utilized in the compatibility test in Section~\ref{sec:comp_test}. Conversely, the statistical Gaussian parameters, which capture the geometric features of the segments, will be used in Section~\ref{sec:soft_ass} to build putative correspondences and Section~\ref{sec:tf_solve} to estimate the transformation.

\begin{figure}[tbp]
	\centering
	\subfigure[]{	
		\includegraphics[width=8cm]{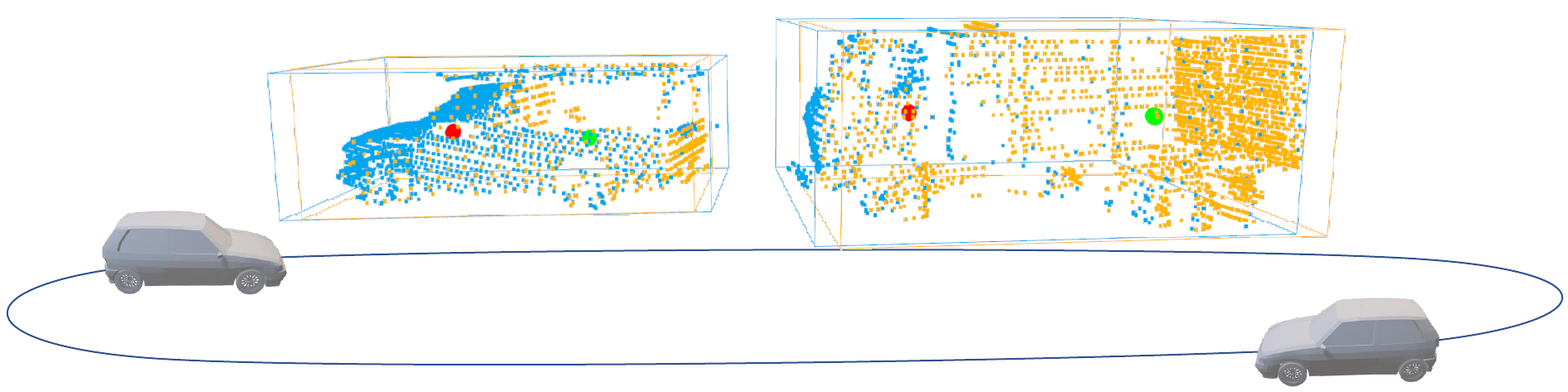}}
	\subfigure[]{	
		\includegraphics[width=8cm]{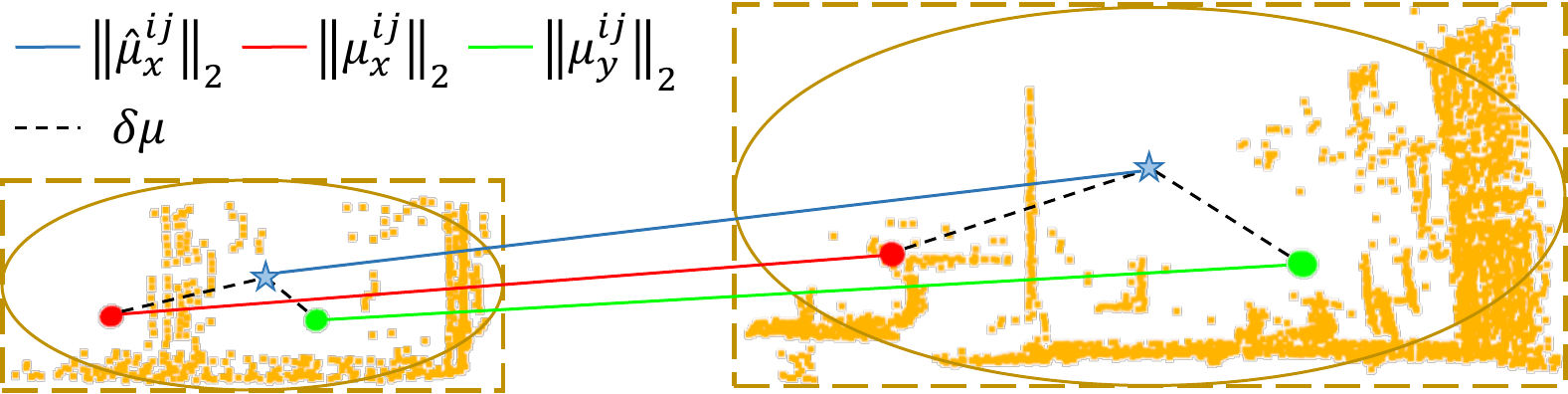}}
	\caption{(a) The blue and orange point clouds are from a pair of loop closure frames. The figure shows two clusters and their centers (the red and green spheres represent the centers of the blue and yellow clusters, respectively) with 3D OBBs. (b) The red and green spheres are defined the same as in (a), and the blue pentagrams are the ground truth centers of the two clusters. The deviations of the observed centers from the ground truth are constrained within OBEs tangent to the OBBs. In addition, the translation and rotation invariant measurements (TRIMs) $\|\mu^{ij}_x{\|_{2}}$ and $\|\mu^{ij}_y{\|_{2}}$ and their ground truths $\|\hat{\mu}_{x}^{ij}{\|_{2}}$ are denoted by the red, green and blue lines, respectively.}
    \label{fig:gem_obb}
\end{figure}

\subsection{Soft Association}
\label{sec:soft_ass}

Starting with a collection of GEMs extracted from the point cloud, we initially select the top-J GEMs from each type of PCL segment. This selection process is based on respective measures of area, volume, and length, as larger segments generally demonstrate superior stability. Following this, we employ a mutual-K-nearest neighbor (MKNN) matching strategy for GEMs that share labels, generating the initial set of putative correspondences represented as $\mathcal{I}$. This MKNN strategy aids in enhancing the likelihood of including true correspondences. The distance between a pair of GEMs $(x, y)$ utilized in the MKNN is determined by the Wasserstein distance of their statistical covariance matrices,

\begin{equation}
\begin{aligned}
\mu_x^{\prime}&=\mathbf{R}^* \mu_x+\mathbf{t}^*\\
\Sigma_x^{\prime}&=\mathbf{R}^* \Sigma_x{\mathbf{R}^{*}}^{\mathrm{T}}
\end{aligned}
\end{equation}

\begin{equation}
\begin{aligned}
W&\left(f\left(x \mid \mathbf{R}^*, \mathbf{t}^*\right), y\right)=\left\|\mu_x^{\prime}-\mu_y\right\|_2^2\\
&+\operatorname{Tr}\left(\Sigma_x^{\prime}+\Sigma_y-2\left(\left(\Sigma_x^{\prime}\right)^{\frac{1}{2}} \boldsymbol{\Sigma}_y\left(\Sigma_x^{\prime}\right)^{\frac{1}{2}}\right)^{\frac{1}{2}}\right)
\end{aligned}
\end{equation}

In an ideal scenario, the optimal rigid transformation $(\mathbf{R}^*, \mathbf{t}^*)$ is ascertained via Equation~\eqref{equ:sub_reg}. However, because it is infeasible during the data association step, we instead derive its substitute $\left(R_j^o R_i^{o T}, \mu_y-R_j^o R_i^{o T} \mu_x\right)$ by aligning the two GEMs.

Our proposed MKNN-based association method not only eliminates the need to devise complex descriptors—which are typically employed for one-to-one matching—but also avoids the quadratic growth associated with all-to-all matching. The effectiveness of the Wasserstein distance is validated through comparisons with other metrics in Section~\ref{exp:front_end}.

\section{Graph-theoretic Outlier Pruning}
\label{sec:graph}

In this section, we introduce an outlier pruning method based on a pyramid compatibility graph. This approach leverages the compatibility graph (Section~\ref{sec:graph_build}), which captures the inter-compatibility among putative correspondences. We assume that inliers will form the maximum clique within this graph. Our framework, operating on a distrust-and-verify scheme, utilizes a multi-threshold compatibility test (Section~\ref{sec:comp_test}) to construct a pyramid compatibility graph. Subsequently, a graduated maximum clique solver is introduced (Section~\ref{sec:mac_solver}) to identify the maximum cliques (MACs) at each level and construct multiple inlier sets. These sets are then employed to generate an array of transformation candidates (Section~\ref{sec:tf_solve}).

\subsection{Compatibility Graph Construction}
\label{sec:graph_build}

Given a set of GEM correspondences $\mathcal{I}=\{(x_k,y_k)\}$, the centers $(\mu_x^k,\mu_y^k)$ of each GEM correspondence follows the following model,
\begin{equation}
\begin{aligned}
\mu_y^k&=\mathbf{R}^* \mu_x^k+\mathbf{t}^*+\mathbf{o}_k+\boldsymbol{\epsilon}_k\\
\mu_x^k &\sim \mathcal{N}\left(\hat{\mu}_x^k, \hat{\Sigma}_x^k\right), \mu_y^k \sim \mathcal{N}\left(\hat{\mu}_y^k, \hat{\Sigma}_y^k\right)\\
\boldsymbol{\epsilon}_k &\sim \mathcal{N}\left(0, \mathbf{R}^* \hat{\Sigma}_x^k {\mathbf{R}^{*}}^{\mathrm{T}}+\hat{\Sigma}_y^k\right)
\end{aligned}
\end{equation}
where $\mu$, $\hat{\mu}$ and $\hat{\Sigma}$ represent the center, true center and pseudo covariance matrix of the GEM, respectively. For inlier correspondences, $\mathbf{o}_k$ is a zero vector, while for outlier correspondences, it is an arbitrary vector. $\boldsymbol{\epsilon}_k$ captures the uncertainty of the GEM center resulting from viewpoint changes and occlusion.

For any pair of correspondences $(x_i,y_i)$ and $(x_j,y_j)$, we can construct a pairwise invariant using two \revised{translation and rotation invariant measurements (TRIMs\cite{yang2020teaser})} for the compatibility test. TRIMs are defined as $\|\mu_x^i-\mu_x^j\|_2$ and $\|\mu_y^i-\mu_y^j\|_2$, exploiting the length-preserving property of rigid transformation, as shown in Figure \ref{fig:gem_obb}. The two correspondences pass the compatibility test if their pairwise invariant satisfies the condition:
\begin{equation}
\left|\left\|\mu_x^i-\mu_x^j\right\|_2-\left\|\mu_y^i-\mu_y^j\right\|_2\right|<\delta_{ij}
\end{equation}
where $\delta_{ij}$ is a threshold used to reject potential outliers in this correspondence pair. It is determined based on the pseudo covariance of the four involved GEMs and will be analyzed in detail in Section \ref{sec:comp_test}.

Finally, we construct a compatibility graph by representing each correspondence as a vertex and connecting an edge between two vertices if they pass the compatibility test.

\setcounter{algocf}{0}
\begin{algorithm}[htbp]
	\SetAlgoLined
    \label{alg:graph_build}
    \caption{Compatibility Graph Construction}
    \SetKwInOut{Input}{Input}
    \SetKwInOut{Output}{Output}
    \SetKwInOut{Param}{Params}
	\Param{Pre-defined $\mathtt{p}$-values $p_1 \geq p_2 \geq \ldots \geq p_M$}
    \Input{Initial putative GEM correspondences $\mathcal{I}=\{\mathtt{(x_k,y_k)|k=1:N}\}$}
    \Output{Pyramid graph $\{\mathtt{\mathcal{G}_m|m=1:M}\}$}  

    \SetKwFunction{graphBuild}{$\mathtt{ConstructGraphs}$}
    \SetKwFunction{UpperBound}{$\mathtt{UpperEigenvalue}$}
    \SetKwFunction{addedge}{$\mathtt{addEdge}$}
	\SetKwFunction{return}{$\mathtt{return}$}
    \SetKwProg{Fn}{Function}{}{}
	\Fn{\graphBuild{$\mathcal{I}$}}{
        \% parallel running\\
        \For{$\mathtt{(i,j)~in~}$$\mathtt{\left(\begin{matrix}N\\2\\\end{matrix} \right)}$}{
            $\mathtt{d_{ij}=\left| \|\mu_x^i-\mu_x^j{{\|}_{2}}-\|\mu_y^i-\mu_y^j{{\|}_{2}} \right|}$; \label{alg:trims} \\
            $\mathtt{\lambda _{x,1}^{ij}=}$\UpperBound{$\mathtt{\hat{\Sigma }_{x}^{i}}$, $\mathtt{\hat{\Sigma }_{x}^{j}}$};\\
            $\mathtt{\lambda _{y,1}^{ij}=}$\UpperBound{$\mathtt{\hat{\Sigma }_{y}^{i}}$, $\mathtt{\hat{\Sigma }_{y}^{j}}$}; \label{alg:ub}\\
            \For{$\mathtt{m=1:M}$}{
                $\mathtt{\delta_{ij}=\sqrt{\chi _{(p_m)}^{2}\lambda _{x,1}^{ij}}+\sqrt{\chi _{(p_m)}^{2}\lambda _{y,1}^{ij}}}$;\\
                \If{$\mathtt{d_{ij}} \leq \mathtt{\delta_{ij}}$}{
                    \For{$\mathtt{k=m:M}$}{
                        $\mathtt{\mathcal{G}_k}.$\addedge{$\mathtt{i}$,$\mathtt{j}$};
                    }
                    $\mathtt{break}$;
                }
            }
        }
	}
	\textbf{End Function}
 
    \SetKwProg{Fn}{Function}{}{}
	\Fn{\UpperBound{$\mathtt{\Sigma_1,\Sigma_2}$}}{\label{alg:rho_ub}
        $\mathtt{\Sigma=\Sigma_1+\Sigma_2}$;\\
        \%Perron Frobenius theorem\\
        $\mathtt{ub_1=\max \left( \sum_{i=1}^{3} |\Sigma_{i,j}| \right) \quad \forall j \in \{1,2,3\}}$;\\
        \%Wolkowicz's method\\
        $\mathtt{m \gets \frac{\text{trace}(\Sigma)}{3}}$;
        $\mathtt{s^2 \gets \frac{\text{trace}(\Sigma^2)}{3} - m^2}$;\\
        $\mathtt{ub_2=m + \sqrt{2s^2}}$;\\
        \%Weyl's inequality\\
        $\mathtt{ub_3=\lambda_{1}{(\Sigma_1)}+\lambda_{1}{(\Sigma_2)}}$;\\
        \return{$\mathtt{\min(ub_1,ub_2,ub_3)}$}
	}
	\textbf{End Function}
\end{algorithm}

\subsection{Multi-Threshold Compatibility Test}
\label{sec:comp_test}

In this section, we delve into the methodology for selecting the appropriate $\delta_{ij}$, an essential parameter that helps ascertain whether a pair of correspondences meet the criteria of the compatibility test. As previously detailed in Section \ref{sec:para_est}, we have utilized pseudo Gaussian distribution parameters to effectively model the uncertainty inherent in the center $\mu$. This leads us to the following relationships:
\begin{equation}
\begin{aligned}
 &{{\mu}_{x}^i}\sim \mathcal{N}\left( \hat{\mu }_{x}^{i},\hat{\Sigma }_{x}^{i} \right), {{\mu}_{x}^j}\sim \mathcal{N}\left( \hat{\mu }_{x}^{j},\hat{\Sigma }_{x}^{j} \right) \\ 
 &{{\mu}_{y}^i}\sim \mathcal{N}\left( \hat{\mu }_{y}^{i},\hat{\Sigma }_{y}^{i} \right), {{\mu}_{y}^j}\sim \mathcal{N}\left( \hat{\mu }_{y}^{j},\hat{\Sigma }_{y}^{j} \right) \\ 
 & {{\mu}_{x}^i}-{{\mu}_{x}^j}\sim \mathcal{N}\left( \hat{\mu }_{x}^{i}-\hat{\mu }_{x}^{j},\hat{\Sigma }_{x}^{i}+\hat{\Sigma }_{x}^{j} \right)=\mathcal{N}\left( \hat{\mu }_{x}^{ij},\hat{\Sigma }_{x}^{ij} \right) \\ 
 & {{\mu}_{y}^i}-{{\mu}_{y}^j}\sim \mathcal{N}\left( \hat{\mu }_{y}^{i}-\hat{\mu }_{y}^{j},\hat{\Sigma }_{y}^{i}+\hat{\Sigma }_{y}^{j} \right)=\mathcal{N}\left( \hat{\mu }_{y}^{ij},\hat{\Sigma }_{y}^{ij} \right) \\ 
\end{aligned}
\end{equation}

Let $\mu^{ij}_x={{\mu}_{x}^i}-{{\mu}_{x}^j}=\hat{\mu }_{x}^{ij}+\epsilon _{x}^{ij}$, where $\epsilon _{x}^{ij}\sim \mathcal{N}\left( 0,\hat{\Sigma }_{x}^{ij} \right)$, and $\mu^{ij}_y={{\mu}_{y}^i}-{{\mu}_{y}^j}=\hat{\mu }_{y}^{ij}+\epsilon _{y}^{ij}$, where $\epsilon _{y}^{ij}\in N\left( 0,\hat{\Sigma }_{y}^{ij} \right)$. We consider the two correspondences be both inliers, thus $\hat{\mu }_{y}^{ij}={{\mathbf{R}}^{*}}\hat{\mu }_{x}^{ij}$. According to the triangle inequality, we have 
\begin{equation}
\label{eq:1}
    \|\mu^{ij}_x{{\|}_{2}}\le \|\hat{\mu }_{x}^{ij}{{\|}_{2}}+\|\epsilon _{x}^{ij}{{\|}_{2}}
\end{equation}
in which $\epsilon _{x}^{ij}$ characterizes the joint uncertainty between the GEMs $\mu_x^i$ and $\mu_x^j$. This joint uncertainty can also be similarly confined within its corresponding uncertainty ellipsoid. To simplify matters and maintain generality, we adopt the probability $p$ as a means to model this uncertainty, as detailed below:

\begin{equation}
\label{eq:2}
    {{(\epsilon _{x}^{ij})}^{\mathrm{T}}}{{(\hat{\Sigma }_{x}^{ij})}^{-1}}\epsilon _{x}^{ij}\le \chi _{(p)}^{2}
\end{equation}
where $\hat{\Sigma }_{x}^{ij}$ is a symmetric and positive-definite matrix that can be expressed in a diagonalized form, as illustrated below:
\begin{equation}
\hat{\Sigma }_{x}^{ij}=U\left[ \begin{matrix}
   \lambda _{x,1}^{ij} & 0 & 0  \\
   0 & \lambda _{x,2}^{ij} & 0  \\
   0 & 0 & \lambda _{x,3}^{ij}  \\
\end{matrix} \right]{{U}^{\mathrm{T}}}
\end{equation}
where $U$ is an orthogonal matrix and $\lambda _{x,1}^{ij} \geq \lambda _{x,2}^{ij} \geq \lambda _{x,3}^{ij}$. Furthermore, we can get a lower bound on ${{(\epsilon _{x}^{ij})}^{\mathrm{T}}}{{(\hat{\Sigma }_{x}^{ij})}^{-1}}\epsilon _{x}^{ij}$,
\begin{equation}
\label{eq:3}
\begin{aligned}
&{{(\epsilon _{x}^{ij})}^{\mathrm{T}}}{{(\hat{\Sigma }_{x}^{ij})}^{-1}}\epsilon _{x}^{ij}\\
&={{({{U}^{\mathrm{T}}}\epsilon _{x}^{ij})}^{\mathrm{T}}}{{\left[ \begin{matrix}
   \lambda _{x,1}^{ij} & 0 & 0  \\
   0 & \lambda _{x,2}^{ij} & 0  \\
   0 & 0 & \lambda _{x,3}^{ij}  \\
\end{matrix} \right]}^{-1}}({{U}^{\mathrm{T}}}\epsilon _{x}^{ij})\\
&=\sum\limits_{k}^{3}{\frac{1}{\lambda _{xk}^{ij}}{{(U^{\mathrm{T}}\epsilon _{x}^{ij})_k}^{2}}}\ge \frac{1}{\lambda _{x,1}^{ij}}\sum\limits_{k}^{3}{{{(U^{\mathrm{T}}\epsilon _{x}^{ij})_k}^{2}}}\\
&=\frac{1}{\lambda _{x,1}^{ij}}{{(\epsilon _{x}^{ij})}^{\mathrm{T}}}\epsilon _{x}^{ij} = \frac{1}{\lambda _{x,1}^{ij}} \|\epsilon _{x}^{ij}{{\|}_{2}^2}\\ 
\end{aligned}
\end{equation}
where $\lambda _{x,1}^{ij}$ is the largest eigenvalue of $\hat{\Sigma }_{x}^{ij}$. By consolidating equations \eqref{eq:1}, \eqref{eq:2} and \eqref{eq:3}, we can derive an upper bound for $\|\mu^{ij}_x\|{2}$, as follows:
\begin{equation}
\label{eq:4}
\begin{aligned}
 \|\mu^{ij}_x{{\|}_{2}}&\le \|\hat{\mu }_{x}^{ij}{{\|}_{2}}+\|\epsilon _{x}^{ij}{{\|}_{2}}\\
 &\le \|\hat{\mu }_{x}^{ij}{{\|}_{2}}+\sqrt{\chi _{(p)}^{2}\lambda _{x,1}^{ij}} \\ 
\end{aligned}
\end{equation}

Following a similar derivation, we can establish the upper bound of $|\mu^{ij}_y|_2$ as follows,
\begin{equation}
\label{eq:5}
\begin{aligned}
 \|\mu^{ij}_y{{\|}_{2}}&\le \|\hat{\mu }_{y}^{ij}{{\|}_{2}}+\|\epsilon _{y}^{ij}{{\|}_{2}}\\
 &\le \|{{\mathbf{R}}^{*}}\hat{\mu }_{x}^{ij}{{\|}_{2}}+\sqrt{\chi _{(p)}^{2}\lambda _{y,1}^{ij}}\\
 &=\|\hat{\mu }_{x}^{ij}{{\|}_{2}}+\sqrt{\chi _{(p)}^{2}\lambda _{y,1}^{ij}} \\ 
\end{aligned}
\end{equation}

Finally, by integrating equations \eqref{eq:4} and \eqref{eq:5}, we can determine $\delta_{ij}$ by assigning a confidence probability $p$ (e.g., 90\%) to accept the observation. It is noteworthy that a larger value of $p$ implies a reduction in uncertainty.
\begin{equation}
 \left| \|\mu^{ij}_x{{\|}_{2}}-\|\mu^{ij}_y{{\|}_{2}} \right|\le \sqrt{\chi _{(p)}^{2}\lambda _{x,1}^{ij}}+\sqrt{\chi _{(p)}^{2}\lambda _{y,1}^{ij}}={{\delta }_{ij}} \\ 
\end{equation}

The comprehensive process of constructing the pyramid graph using a multi-threshold compatibility test is encapsulated in Algorithm \ref{alg:graph_build}. For $N$ correspondences, it is necessary to compute the TRIMs and their upper bounds for $N(N-1)/2$ pairs, as indicated in lines \ref{alg:trims}-\ref{alg:ub} of Algorithm \ref{alg:graph_build}. \revised{In the worst case, the time complexity reaches \(O(M \cdot N^2)\), accounting for the majority of the computational effort in the PAGOR.} To reduce the computational burden associated with eigenvalue calculations, we employ the Perron-Frobenius theorem\cite{perron1907theorie,frobenius1912matrizen}, Wolkowicz's method\cite{wolkowicz1980bounds}, and Weyl's inequality\cite{weyl1912asymptotische}. These techniques aid in computing the upper bound of the largest eigenvalue $\lambda_1$, as presented in line \ref{alg:rho_ub}, thereby accelerating the computation process.

\subsection{Graduated MAC Solver}
\label{sec:mac_solver}

Given $M$ $p$-values $p_1 \geq p_2 \geq \ldots \geq p_M$, we correspondingly have $\chi^2$-values $\chi_{(p_1)}^{2} \leq \chi_{(p_2)}^{2} \leq \ldots \leq \chi_{(p_M)}^{2}$. From these, we derive $M$ compatibility graphs $\mathcal{G}$ along with their associated maximum cliques ${\mathcal{G}^\text{MC}}$. Theoretically, as the value of $\chi_{(p)}^{2}$ increases, $\delta_{ij}$ follows the change, which results in more pairwise correspondences passing the compatibility test, thereby leading to denser compatibility graphs $\mathcal{G}$. We arrange all the consistency graphs $\mathcal{G}$ in the order of their sparsity, placing the sparsest at the top and the densest at the bottom. Our objective is to compute these $M$ maximum cliques $\mathcal{G}^\text{MC}$, where their vertices $\mathcal{I}^{*}_m$ are potential inlier sets of the initial correspondences $\mathcal{I}$. At each level of the pyramid graph, our goal is to solve the following optimization problem:

\begin{equation}
\begin{aligned}
   &\underset{{{\mathcal{I}}^{*}_m}\subseteq \{1,\ldots ,N\}}{\mathop{\operatorname{maximize}}}\, |{{\mathcal{I}}^{*}_m}|  \\
   &\text{s.t. } \left| \|\mu_x^i-\mu_x^j{{\|}_{2}}-\|\mu_y^i-\mu_y^j{{\|}_{2}} \right| < {{\delta }_{ij}^m},\forall i,j\in {{\mathcal{I}}^{*}_m}  
\end{aligned}
\end{equation}

\setcounter{theorem}{0}
\begin{proposition}[Lower Bound of Clique Cardinality]
\label{pro:maxClique}
Given $\chi_{(p_{m+1})}^{2} > \chi_{(p_{m})}^{2}$, it follows that $|{{\mathcal{G}}^{MC}_{m+1}}| \geq |{{\mathcal{G}}^{MC}_m}|$.
\end{proposition} 

\begin{proof}
Proceeding by contradiction, let us assume the contrary: $\chi_{(p_{m+1})}^{2}>\chi_{(p_{m})}^{2}$, but $|{{\mathcal{G}}^{MC}_{m+1}}|<|{{\mathcal{G}}^{MC}_m}|$. Let us denote the consistency thresholds corresponding to $\chi_{(p_m)}^2$ and $\chi_{(p_{m+1})}^2$ as $\delta_{ij}^m$ and $\delta_{ij}^{m+1}$ respectively. Given that $\chi_{(p_{m+1})}^2 > \chi_{(p_m)}^2$, we can infer that $\delta_{ij}^{m+1} > \delta_{ij}^m$ for all $i,j$.

Consider the maximum cliques at levels $m$ and $m+1$:

\begin{equation}
\begin{aligned}
&|{{\mathcal{G}}^{MC}_{m}}|=|{{\mathcal{I}}^{*}_{m}}| \\
&\text{s.t. } \left| \|\mu_x^i-\mu_x^j{{\|}_{2}}-\|\mu_y^i-\mu_y^j{{\|}_{2}} \right| < \delta_{ij}^{m},~\forall i,j \in {\mathcal{I}^{*}_{m}}
\end{aligned}
\end{equation}

\begin{equation}
\begin{aligned}
&|{{\mathcal{G}}^{MC}_{m+1}}|=|{{\mathcal{I}}^{*}_{m+1}}| \\
&\text{s.t. } \left| \|\mu_x^i-\mu_x^j{{\|}_{2}}-\|\mu_y^i-\mu_y^j{{\|}_{2}} \right| < \delta_{ij}^{m+1},~\forall i,j \in {\mathcal{I}^{*}_{m+1}}
\end{aligned}
\end{equation}

Given that $\delta_{ij}^{m+1} > \delta_{ij}^m$, any feasible set for ${\mathcal{I}}^{*}_m$ must also be a feasible set for ${\mathcal{I}}^{*}_{m+1}$. Hence, the cardinality of $|{{\mathcal{G}}^{MC}_{m+1}}|$ must be greater than or equal to that of $|{{\mathcal{G}}^{MC}_{m}}|$. This contradicts our earlier assumption that $|{{\mathcal{G}}^{MC}_{m+1}}| < |{{\mathcal{G}}^{MC}_m}|$, and consequently, our proof is concluded.
\end{proof}

Proposition \ref{pro:maxClique} suggests that we can leverage the maximum clique deduced from a sparse clique to produce a lower bound on the cardinality of MAC for the denser compatibility graph of the subsequent level. This insight works well with an efficient parallel maximum clique finder algorithm, PMC \cite{rossi2015parallel}, enabling us to devise a graduated PMC algorithm that calculates the maximum clique from the top to the bottom of the pyramid compatibility graph. Specifically, the original PMC algorithm is based on a branch-and-bound method and prunes using the core numbers of the vertices. In the graduated PMC, once we have determined the maximum clique from the sparser graph of the preceding layer, we can exclude all nodes in this layer whose degree is less than the cardinality of the clique. This process substantially reduces the search space, as these nodes cannot be included in the final MAC of the current layer, according to Proposition \ref{pro:maxClique}. 

\revised{Although finding the maximum clique typically has exponential computational complexity in the worst-case scenario, this is mitigated by the sparsity of the compatibility graph under high outlier ratios. Combined with parallel computing techniques, this sparsity ensures that the maximum clique problem can be solved efficiently in practice. Moreover, as we demonstrate empirically in Section~\ref{exp:ab_study}, the integration of an additional compatibility graph does not significantly increase computation time when using the graduated MAC solver, due to the prior lower bound provided by the preceding graph.}

\section{Robust Transformation Estimation}
\label{sec:solve_verify}

In this section, we illustrate how to derive multiple transformation candidates from the $M$ hypothetical inlier sets $\mathcal{I}^{*}$ utilizing robust estimation methods (Section \ref{sec:tf_solve}). We then examine the selection of the evaluation function $g$ in Eq. \ref{eq:eval} from a robust statistics viewpoint, with the goal of identifying the most appropriate transformation from the set of candidates that can optimally align $\mathcal{X}$ and $\mathcal{Y}$ (Section \ref{sec:verify}).

\subsection{Distribution-to-Distribution Registration}
\label{sec:tf_solve}

Given a hypothetic inlier set $\mathcal{I}^{*}=(x_k,y_k)$, we reformulate the Equation \eqref{equ:sub_reg} as a distribution-to-distribution registration problem using the statiscal Gaussian parameters of GEMs. Let $\mu_x^k \in \mathcal{N}(\hat{\mu}_x^k, \Sigma_{x}^k)$ and $\mu_y^k \in \mathcal{N}(\hat{\mu}_y^k, \Sigma_{y}^k)$, Equation \eqref{equ:sub_reg} can be written as,

\begin{equation}
\label{equ:d2d}
\begin{aligned}
\mathbf{R}^*, \mathbf{t}^*&=\underset{\mathbf{R} \in \operatorname{SO}(3), \mathbf{t} \in \mathbb{R}^3}{\arg \min } \sum_{\left(\mu_x^k, \mu_y^k\right) \in \mathcal{I}^*} \min \left(r\left(\mu_y^k, \mu_x^k\right), \bar{c}^2\right) \\
&r\left(\mu_y^k, \mu_x^k\right)= d_{k}^{\mathrm{T}}\left(\Sigma_{y}^k+ \mathbf{R} \Sigma_{x}^k {\mathbf{R}}^{\mathrm{T}} \right)^{-1} d_{k}\\
&d_{k}=\mu_y^k-\left(\mathbf{R} \mu_x^k+\mathbf{t}\right) \\
\end{aligned}
\end{equation}
where $d_{k}\sim \mathcal{N}\left(\hat{\mu}^k_y-(\mathbf{R} \hat{\mu}_x^k+\mathbf{t}), \Sigma_{y}^k+ \mathbf{R}\Sigma_{x}^k {\mathbf{R}}^{\mathrm{T}} \right)$. $\rho(\cdot)$ is now defined as the Truncated Least Squares (TLS) cost, with $\bar{c}$ representing the inlier cost threshold which is initializad by $\chi_{(0.01)}^2$. The minimization of the residual function $r(\cdot)$ is equivalent to maximizing the log-likelihood of the new Gaussian model $d_{k}$. 

To solve Equation \eqref{equ:d2d}, we use graduated non-convexity\cite{yang2020graduated} implemented in GTSAM\cite{gtsam} where the non-minimal global solver is replaced by Levenberg-Marquardt optimization for the reason that there is no closed-form solution for distribution-to-distribution registration problem. In addition, for plane segments, the covariance matrix is regularized by replacing its eigenvalue with $(1, 1, 10^{-3})$ to work as a plane-to-plane registration. \revised{GNC acts as an alternative optimization with inner and outer iteration for L-M optimization and weight update, respectively.}

\revised{The time complexity of the optimization process is denoted as \(O(D^2N \cdot L)\), where \(D=6\) represents the dimension of the optimization variable, \(N\) denotes the number of inliers denoted by \(|\mathcal{I}^*|\), and \(L\) indicates the number of outer iterations. In practical applications, the GEM-based matching algorithm tends to yield a limited number of inliers, which keeps \(N\) relatively small. Additionally, the use of maximum clique inlier selection effectively eliminates the majority of outliers, resulting in a robust initial guess. This, in turn, reduces \(L\), the count of outer iterations required for convergence. Consequently, both \(N\) and \(L\) contribute to a lower practical time complexity, enhancing the overall efficiency of the optimization process, which is demonstrated in Figure \ref{figure:runtime}.}

\subsection{Transformation Verification}
\label{sec:verify}

Given M hypothetical transformations $(\mathbf{R}^{*}_{m}, \mathbf{t}^{*}_{m})$, our objective is to construct the evaluation function $g(\cdot)$ to identify the most suitable transformation by resolving Equation \eqref{eq:eval}. Whereas existing approaches endeavor to  \revised{certify the global optimality\cite{holmes2023efficient,yang2020teaser,yang2022certifiably}} of minimizing the objective function, we take a step back to re-think how to design an objective function that accurately reflects the alignment quality of the point clouds. We argue the objective in Equation \eqref{equ:d2d} may be inadequate for two reasons. First, although GEMs recover geometries of $\mathcal{X}$ and $\mathcal{Y}$ maximally, information loss persists since $\mathcal{I}^{*}$ contains only a subset of all potential correspondences. Second, $\mathcal{I}^{*}$ unavoidably contains adversarial outliers that the TLS cannot accommodate.

To address this, we leverage the semantic voxel map $V$ from Section \ref{sec:gmm_front} to verify candidates, as it retains most raw geometric information. Assuming every $(\mathbf{R}^{*}_{m}, \mathbf{t}^{*}_{m})$ is optimal, correspondences producing large residuals are considered outliers. Thus, we formulate $g(\cdot)$ based on the Chamfer distance with a new robust kernel $\hat{\rho}(\cdot)$:

\begin{equation}
\begin{aligned}
g\left(\mathbf{R}_{m}^{*}, \mathbf{t}_{m}^{*} \mid \mathcal{X}, \mathcal{Y} \right)& = \frac{1}{|\mathcal{X}|} \sum_{v_x \in V_{\mathcal{X}}} \hat{\rho}((r(v_x^{\prime},v_y^{\prime}))\\
v_x^{\prime}&=\mathbf{R}_{m}^{*} v_x + \mathbf{t}_{m}^{*}\\
v_y^{\prime}&=\underset{v_y \in V_{\mathcal{Y}}}{\arg\min} \|v_x^{\prime} - v_y\|_2
\end{aligned}
\end{equation}
where $V_{\mathcal{X}}$ and $V_{\mathcal{Y}}$ represent the voxel maps of $\mathcal{X}$ and $\mathcal{Y}$.

\revised{To expedite the computational process, we compress the raw point clouds by retaining only the sampled five points within each voxel $V_{\mathcal{X}}$ and the centroid of each voxel $V_{\mathcal{Y}}$. For each sampled point $v_x$ in $V_{\mathcal{X}}$, we determine its nearest counterpart $v_y^{\prime}$ in $V_{\mathcal{Y}}$ after transforming it with the optimal rotation $\mathbf{R}_{m}^{*}$ and translation $\mathbf{t}_{m}^{*}$. We then calculate the residual $r(v_x^{\prime}, v_y^{\prime})$ as follows:}

\begin{equation} 
\begin{cases}
|n_y^{\mathrm{T}}(v_x^{\prime}-v_y^{\prime})| & \text{if } \text{type}(v_y^{\prime}) = \text{plane} \\
\left\|\left(\mathbf{I}-d_y^{\mathrm{T}} d_y\right)(v_x^{\prime}-v_y^{\prime})\right\|_2 & \text{if } \text{type}(v_y^{\prime}) = \text{line} \\
\|v_x^{\prime}-v_y^{\prime}\|_2 & \text{if } \text{type}(v_y^{\prime}) = \text{cluster}
\end{cases}
\end{equation}

\revised{In this equation, $n_y$ denotes the normal to the plane, and $d_y$ represents the direction of the line. The nearest neighbor search leverages a KD-Tree structure, where building the tree requires a time complexity of \(O(n \log(n))\) and querying the nearest point has a complexity of \(O(m \log(n))\), with $m$ and $n$ being the numbers of sampled points in $V_{\mathcal{X}}$ and $V_{\mathcal{Y}}$, respectively. Furthermore, the verification of $(\mathbf{R}^{*}_{m}, \mathbf{t}^{*}_{m})$ will be omitted if it is too close to an already evaluated one.}

For the selection of the evaluation function $\hat{\rho}(\cdot)$, we employ Dynamic Covariance Scaling (DCS)\cite{agarwal2013robust} and compare its performance to other robust kernels, including Tukey, Cauchy, Huber\cite{zhang1997parameter}, and TLS. This comparative analysis is presented in Section \ref{exp:ab_study}.
\section{Experiments}
\label{sec:exp}

\begin{figure*}[t]
	\centering
         \subfigure[Roll ($^\circ$)]{	
		\includegraphics[width=4.2cm]{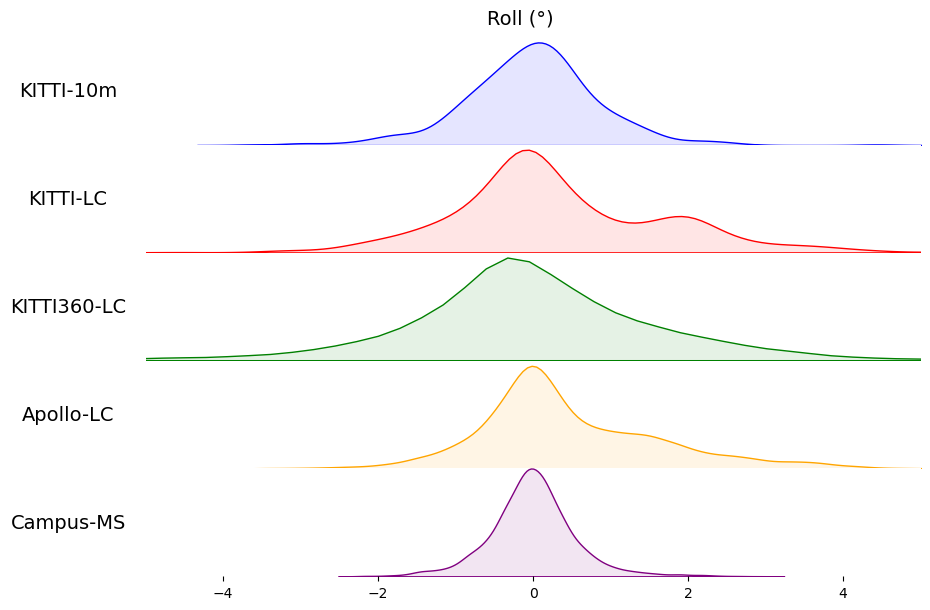}}
        \subfigure[Pitch ($^\circ$)]{	
		\includegraphics[width=4.2cm]{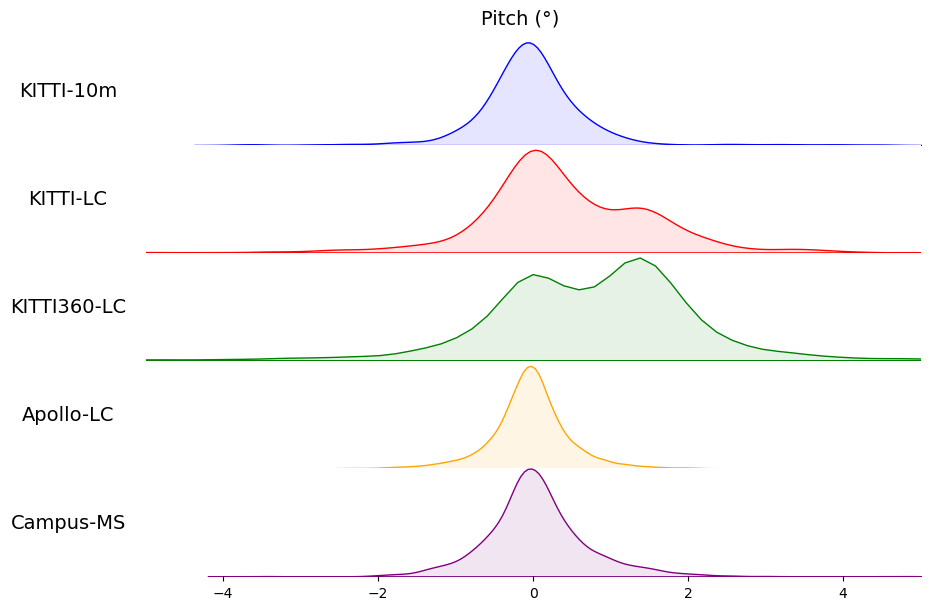}}
        \subfigure[Yaw ($^\circ$)]{	
		\includegraphics[width=4.2cm]{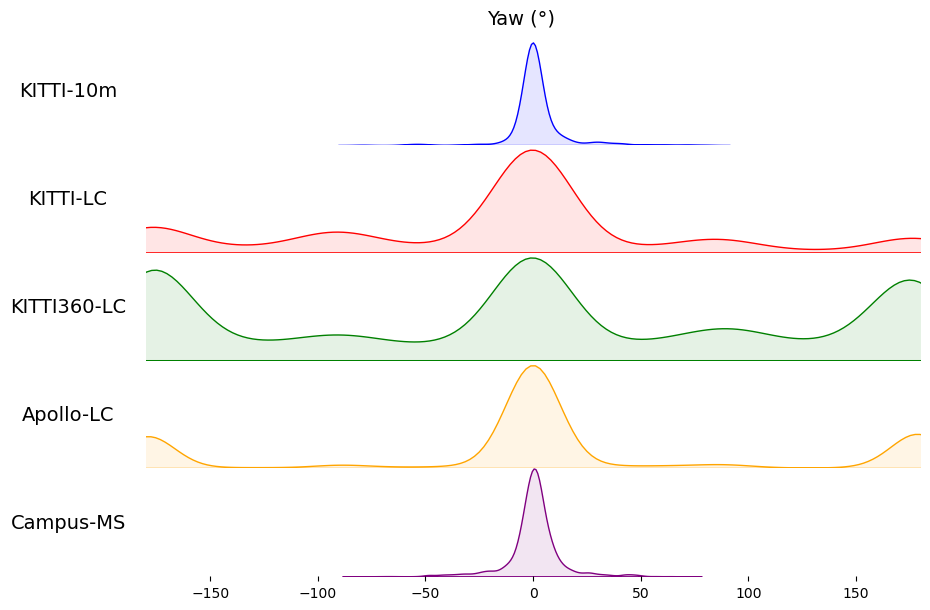}}
        \subfigure[Translation (m)]{	
		\includegraphics[width=4.2cm]{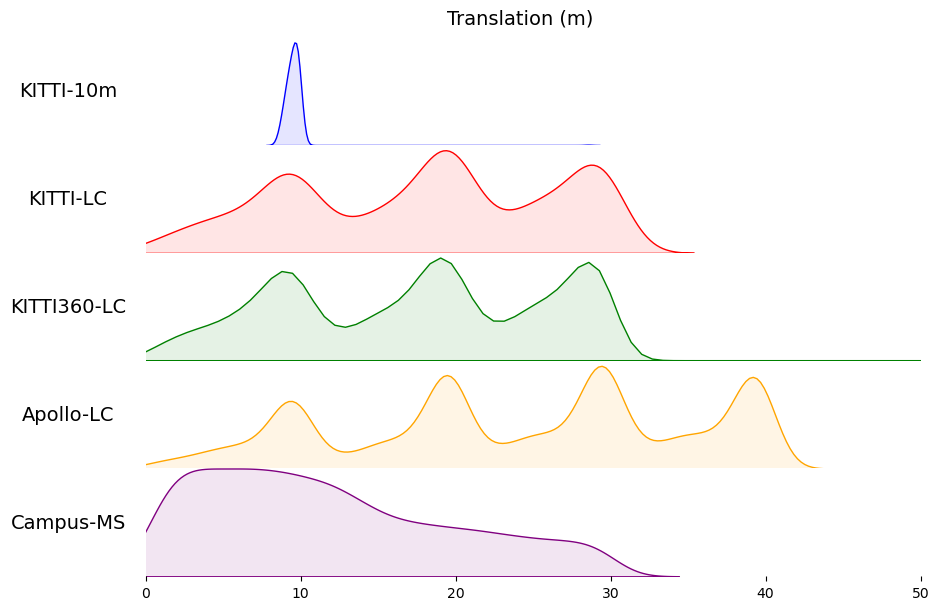}}
	\caption{The distributions of ground truth rotations and translations in different datasets. This figure shows broader distributions that can effectively validate global registration performances. }
	\label{figure:dataset_info}
\end{figure*}

This section conducts a comprehensive comparative analysis between our proposed method and existing SOTA techniques, employing various datasets and sensor types. Our method achieves better performance and delivers sub-optimal results with a diminished computational overhead. Additionally, the integration of our novel sub-modules, namely GEM and PAGOR, into any registration approach demonstrates a substantial enhancement in performance. Finally, a review of selected failure cases is also presented to provide practitioners with an in-depth understanding of our methodology.

\subsection{Experimental Setup}
\label{exp:setup}

\subsubsection{Datset}
\label{sec:dataset}

As presented in Table \ref{tab:dataset_info}, we conduct experiments on three public and one self-collected multi-session dataset, namely Campus-MS. The KITTI-10m registration set, widely utilized for the evaluation of LiDAR-based global registration techniques \cite{choy2020deep, bai2021pointdsc, zhang20233d}, is formed by sampling consecutive frames at 10-meter intervals from sequences 08, 09, and 10 of the KITTI odometry dataset. However,  this dataset's concentrated distribution of relative poses between the two point cloud frames may not comprehensively assess the true performance of global registration. 

\revised{To address this constraint, we adopt a strategy of sampling registration pairs from revisited locations, a testing approach  commonly utilized in re-localization and loop closing scenarios. This approach yields a broader distribution of poses and more pronounced disparities in point cloud data as shown in Figure \ref{figure:dataset_info}. The translation-based down-sampling is also employed to decrease the test set size, while ICP is utilized to refine the ground truth.}

\revised{The strategy outlined above is implemented to process three publicly accessible datasets, including KITTI-LC \cite{geiger2012we}, KITTI360-LC \cite{liao2022kitti}, and Apollo-LC \cite{lu2019deepvcp} (LC means loop closing). The datasets cover a wide range of scenarios pertinent to autonomous navigation, including urban centers, highways, overpasses, and suburban areas.} We have ensured the public availability of the code used to generate this benchmark dataset, with the intention of benefiting the wider community~\footnote{\href{https://github.com/HKUST-Aerial-Robotics/LiDAR-Registration-Benchmark}{https://github.com/HKUST-Aerial-Robotics/LiDAR-Registration-Benchmark}}. The self-collected Campus-MS dataset will be introduced in detail in Section~\ref{sec:real-world}.

\subsubsection{Comparisons}
Concerning the front end, we employ two approaches: the handcrafted descriptor FPFH\cite{rusu2009fast} and the deep learning-based descriptor FCGF\cite{choy2019fully}. \revised{As for the back end, our selection encompasses SOTA methods based on different frameworks: (1) graph-theoretic methods: TEASER++\cite{yang2020teaser}, Quatro\cite{lim2022single}, and 3DMAC\cite{zhang20233d}; (2) deep learning-based methods: DGR\cite{choy2020deep} and PointDSC\cite{bai2021pointdsc}; (3) enhanced RANSAC; (4) advanced loop closing methods that integrate global registration, like STD\cite{yuan2023std}, Cont2\cite{jiang2023contour}.}

\subsubsection{Evaluation metric}
\begin{table}[!]
\centering
\caption{Datasets for Performance Evaluaiton}
\renewcommand{\arraystretch}{1}
\resizebox{\columnwidth}{!}{%
\begin{tabular}{cclllclllccc}
\hline
              & \multicolumn{4}{c}{KITTI-10m}   & \multicolumn{4}{c}{KITTI-loop}  & KITTI-360   & Apollo-SouthBay & Self-collected \\ \hline
Num. of Pairs & \multicolumn{4}{c}{556}         & \multicolumn{4}{c}{3325}        & 18469       & 55118           & 4707           \\
Testing Strategy          & \multicolumn{4}{c}{Loop Clouse} & \multicolumn{4}{c}{Loop Clouse} & Loop Clouse & Loop Clouse     & Multi Sessions \\
Sensor Type         & \multicolumn{4}{c}{Velodyne} & \multicolumn{4}{c}{Velodyne} & Velodyne & Velodyne     & Livox  \\ \hline
\end{tabular}%
}
\label{tab:dataset_info}
\end{table}
\begin{table}[t]
\caption{All Parameters of Our Approach}
\label{tab:paras}
\renewcommand{\arraystretch}{1}
\centering
\begin{tabular}{ccc|c}
\hline
\multicolumn{3}{c|}{Parameters}                                            & Value             \\ \hline
\multicolumn{1}{c|}{\multirow{4}{*}{Plane}} & \multicolumn{1}{c|}{\multirow{2}{*}{Extraction}} & Voxel size $s_v$       & 1 m   \\
\multicolumn{1}{c|}{}   & \multicolumn{1}{c|}{}   & $\lambda_2/\lambda_3$ threshold $\sigma_{p}$    & 30   \\ \cline{2-4} 
\multicolumn{1}{c|}{}   & \multicolumn{1}{c|}{\multirow{2}{*}{Merge}}      & Normal threshold $\sigma_{n}$   & 0.95  \\
\multicolumn{1}{c|}{}   & \multicolumn{1}{c|}{}   & Distance threshold $\sigma_{p2n}$    & 0.2      \\ \hline
\multicolumn{2}{c|}{\multirow{2}{*}{Line extraction}}   & Distance threshold $\sigma_{p2l}$ & 0.5 m \\
\multicolumn{2}{c|}{}                             & Inlier ratio threshold $\sigma_{line}$ & 0.5               \\ \hline
\multicolumn{2}{c|}{\multirow{2}{*}{Association}} & Segment number $J$        & 50                \\
\multicolumn{2}{c|}{}                             & $K$ in MKNN strategy     & 20                \\ \hline
\multicolumn{2}{c|}{MAC solver}                   & p-values        & [0.99,0.95,0.9,0.8] \\ \hline
\multicolumn{2}{c|}{Geometric verification}       & Robust kernel function         & DCS               \\ \hline
\end{tabular}
\end{table}
\begin{table*}[]
\centering
\caption{\revised{Feature Matching Evaluation on KITTI-LC Dataset}}
\label{tab:matching}
\resizebox{\textwidth}{!}{%
\begin{threeparttable}
\begin{tabular}{c|c|cccc|cccc|cccc}
\hline
 &  & \multicolumn{4}{c|}{0-10m} & \multicolumn{4}{c|}{10-20m} & \multicolumn{4}{c}{20-30m} \\ \cline{3-14} 
\multirow{-2}{*}{} & \multirow{-2}{*}{Method} & IR\tnote{1} (\%) $\uparrow$ & Num & Recall (\%)$\uparrow$ & Time (ms) $\downarrow$ & IR (\%)$\uparrow$ & Num & Recall (\%)$\uparrow$ & Time (ms)$\downarrow$ & IR (\%)$\uparrow$ & Num & Recall (\%)$\uparrow$ & Time (ms)$\downarrow$ \\ \hline
 & FPFH & {\ul 9.198} & 642 & \textbf{99.89} & {\color[HTML]{FE0000} 170.69} & {\ul 2.575} & 590 & 94.52 & {\color[HTML]{FE0000} 177.76} & {\ul 0.923} & 557 & 64.44 & {\color[HTML]{FE0000} 175.16} \\
\multirow{-2}{*}{Point-based} & FCGF & \textbf{29.08} & 2762 & 96.72 & 96.21 & \textbf{13.12} & 2586 & 87.84 & 86.52 & \textbf{5.678} & 2403 & 71.59 & 84.12 \\ \hline
 & Plane only & 1.109 & 450 & 73.08 & {\ul 21.69} & 0.580 & 447 & 42.83 & {\ul 21.87} & 0.337 & 441 & 18.73 & {\ul 22.11} \\
 & Line only & 2.678 & 180 & 56.45 & \textbf{20.65} & 1.454 & 185 & 34.23 & \textbf{21.14} & 0.741 & 175 & 17.38 & \textbf{21.25} \\
 & Cluster only & 1.677 & 713 & 98.03 & 22.76 & 0.744 & 714 & 81.49 & 22.77 & 0.356 & 712 & 43.65 & 23.87 \\
 & Ours (Random) & 0.618 & 1737 & 95.18 & 24.43 & 0.342 & 1753 & 80.19 & 24.57 & 0.203 & 1738 & 53.25 & 25.02 \\
 & Ours (IoU3D) & 0.418 & 1379 & 81.62 & 24.97 & 0.270 & 1381 & 61.16 & 24.83 & 0.175 & 1361 & 36.19 & 25.47 \\
 & Ours (EigenVal) & 1.159 & 1379 & 98.25 & 24.98 & 0.562 & 1377 & 91.57 & 25.12 & 0.302 & 1356 & 65.24 & 25.30 \\
 & Ours (FPFH+T) & 1.330 & 1196 & 98.14 & 45.38 & 0.617 & 1203 & 89.23 & 42.58 & 0.323 & 1189 & 60.24 & 43.34 \\
 & Ours (All-to-all) & 0.692 & 3598 & {\ul 99.02} & 24.41 & 0.360 & 3626 & \textbf{96.35} & 24.50 & 0.199 & 3581 & \textbf{85.08} & 25.77 \\
 & Ours (Wass+D) & 1.376 & 1292 & 96.94 & 25.12 & 0.687 & 1291 & 92.61 & 25.10 & 0.376 & 1273 & 71.75 & 25.27 \\
\multirow{-10}{*}{Segment-based} & Ours (Wass+T) & 1.526 & 1336 & 98.69 & 25.02 & 0.737 & 1339 & {\ul 94.53} & 24.80 & 0.391 & 1320 & {\ul 74.52} & 25.17 \\ \hline
\end{tabular}%
\begin{tablenotes}
\item[1] Inlier ratio.
\end{tablenotes}
\end{threeparttable}
}
\end{table*}

\revised{To conduct a thorough evaluation of performance across different levels of difficulty and variations in point cloud distribution, we categorize the test set into three subsets. These subsets are defined by the translational distance between pairs: $[0,10]$ for easy scenarios, $[10,20]$ for medium complexity, and $[20,30]$ for hard cases. These intervals are particularly relevant for robotic applications such as loop closure \cite{uy2018pointnetvlad, liu2019lpd}.}

\revised{For feature matching evaluation in Section \ref{exp:front_end},} the inlier ratio (IR), total number of correspondences, correspondence recall, and time consumption are used. A correspondence $(x_i,y_i)$ is classified as an inlier if it adheres to the following condition:
\begin{equation}
    \|\mathbf{R}_{gt} x_i + \mathbf{t}_{gt}-y_i\|_2 < 0.5
\end{equation}
where $(\mathbf{R}_{gt}, \mathbf{t}_{gt})$ denote the ground truth rotation and translation, respectively. If the count of inliers exceeds 3, that pair is considered successfully recalled. Subsequently, we compute the recall rate for all pairs in the test set.

\revised{For the global registration evaluation in Section \ref{exp:registration}, we use registration recall for evaluation.} Given the estimated transformation $(\mathbf{R}_{est}, \mathbf{t}_{est})$, the registration is classified as successfully recalled if
\begin{equation}
\begin{aligned}
    \arccos \left(\frac{\operatorname{tr}(\mathbf{R}_{est}^{T}\mathbf{R}_{gt})-1}{2}\right) < 5\deg\\
    \| \mathbf{R}_{est}^{T}(t_{gt}-t_{est}) \|_2 < 2\text{m}
\end{aligned}
\end{equation}

The rotation and translation thresholds of 5 degrees and 2 meters, respectively, are consistent with the configurations used by other state-of-the-art methods\cite{choy2020deep,yin2023segregator,zhang20233d}. This metric setting is aligned with the fact that global registration methods typically prioritize robustness, while local registration methods like iterative closest point (ICP)\cite{besl1992method} and normal distribution transformation (NDT)\cite{magnusson2007scan}, emphasize the accuracy of point cloud registration.

\subsubsection{Implementation details}
\label{sec:impl}

\revised{All the experiments are evaluated on an Intel i7-13700K CPU and an Nvidia RTX4080 graphics card. All parameters in GEM and PAGOR are detailed in Table \ref{tab:paras}. Additionally, we maintain uniformity in parameter across all experiments, even when the LiDAR type and scenarios vary. For ground removal and clustering, we adopt the default parameters as outlined in their respective original papers\cite{oh2022travel,zhou2021t}.}

\revised{We optimized FPFH-based matching for both efficiency and effectiveness by performing ground removal, voxel downsampling at a 0.5m resolution, and adjusting the search radii for normals and FPFH to 1m and 2.5m, respectively. For FCGF matching, we used pre-trained models from the KITTI dataset's open-source code. Correspondences of both FPFH and FCGF are filtered by the mutual nearest neighbors (MNN) strategy. In graph-theoretical methods, we fine-tuned key parameters to optimize performance: noise bounds are set at 0.4 m for TEASER++ and Quatro, and the compatibility threshold is 0.9 for 3DMAC. For RANSAC, we use a default of 1,000,000 iterations and employ TRIMs with a high threshold to efficiently filter out incorrect correspondences. Multi-processing acceleration is used for TEASER++, Quatro, 3DMAC and RANSAC.}

\subsection{Feature Matching Evaluation}
\label{exp:front_end}

\begin{table}[]
\centering
\caption{Global Registration Evaluation on KITTI-10m Dataset}
\label{tab:kitti_10m}
\resizebox{\columnwidth}{!}{%
\begin{tabular}{ccccccc}
\hline
\multirow{2}{*}{Front-end} & \multirow{2}{*}{Back-end} & \multicolumn{4}{c}{Recall (\%) on KITTI-10m $\uparrow$} & \multirow{2}{*}{\begin{tabular}[c]{@{}c@{}}Average\\ Time (ms)\end{tabular}} \\ \cline{3-6}
 &  & 08 & 09 & 10 & Total &  \\ \hline
FCGF & DGR & 96.74 & 99.38 & 97.70 & 97.66 & 573.34 \\
FCGF & PointDSC & 98.37 & 99.38 & {\ul 98.85} & 98.74 & 333.78 \\ \hline
FPFH & RANSAC & 97.39 & \textbf{100.0} & 95.40 & 97.84 & 264.31 \\
FPFH & TEASER & 98.04 & 98.76 & 93.10 & 97.48 & 236.35 \\
FPFH & 3DMAC & 95.44 & 98.15 & 90.80 & 95.50 & 258.29 \\
FPFH & PAGOR & \textbf{99.67} & \textbf{100.0} & \textbf{100.0} & \textbf{99.82} & 282.38 \\ \hline
GEM & RANSAC & 83.39 & 86.42 & 70.11 & 82.19 & 64.68 \\
GEM & TEASER & {\ul 99.35} & {\ul 99.38} & 94.25 & 98.56 & \textbf{45.55} \\
GEM & 3DMAC & 99.02 & \textbf{100.0} & 93.10 & 98.38 & {\ul 55.05} \\
GEM & PAGOR & 99.34 & \textbf{100.0} & {\ul 98.85} & {\ul 99.46} & 61.88 \\ \hline
\end{tabular}%
}
\end{table}
\revised{We assess the quality of putative correspondences from our GEM-based front end against FPFH\cite{rusu2009fast} and FCGF\cite{choy2019fully}, using the KITTI-LC dataset. Our evaluation explores various aspects of our method, considering segment types, matching strategies, and clustering algorithms as follows:}

\revised{
\begin{itemize}
    \item Use of individual segment types (plane, cluster, line).
    \item Replacement of the covariance matrix's Wasserstein distance with alternative matchings, such as IoU3D, the eigenvalue-based descriptor\cite{weinmann2014semantic} (EigenVal), matching randomly (Random), and enumerating all possible correspondences (All-to-all). Furthermore, we try to follow FPFH to design a GEM descriptor for encoding the local information of the neighborhood (20m) using GEM's centers and normals (FPFH+T).
    \item Use of different clustering methods: DCVC\cite{zhou2021t} (Wass+D) and TRAVEL\cite{oh2022travel} (Wass+T).
\end{itemize}}

\revised{Table \ref{tab:matching} summarizes our findings. Segment-based matching tends to produce more outliers than point-based due to its less robust descriptors and difficulties in estimating segment centers. However, our PAGOR algorithm mitigates this issue, effectively handling high outlier ratios, which indicates that the recall rate is more important than inlier ratio. Additionally, as the translation increases, the segment-based method exhibits superior recall compared to the descriptor-based approach like FPFH and FCGF, since the change of neighborhood damage the descriptor performance as discussed in Section \ref{sec:rw:des}. For efficiency, we employ the MKNN strategy over the all-to-all strategy, which balances performance with computational practicality. Table~\ref{tab:matching} also presents the contribution of any individual geometry type. Moreover, Wasserstein distance-based matching proves to be the most effective, aided by our MKNN strategy. Finally, TRAVEL clustering slightly outperforms DCVC, but the latter's fewer parameters may be preferable depending on the application context. We use DCVC for solid-state LiDAR and TRAVEL for mechanical LiDAR.}
\begin{figure*}[t]
	\centering
	\subfigure[KITTI-LC (0.45\%)]{	
		\includegraphics[width=0.45\textwidth]{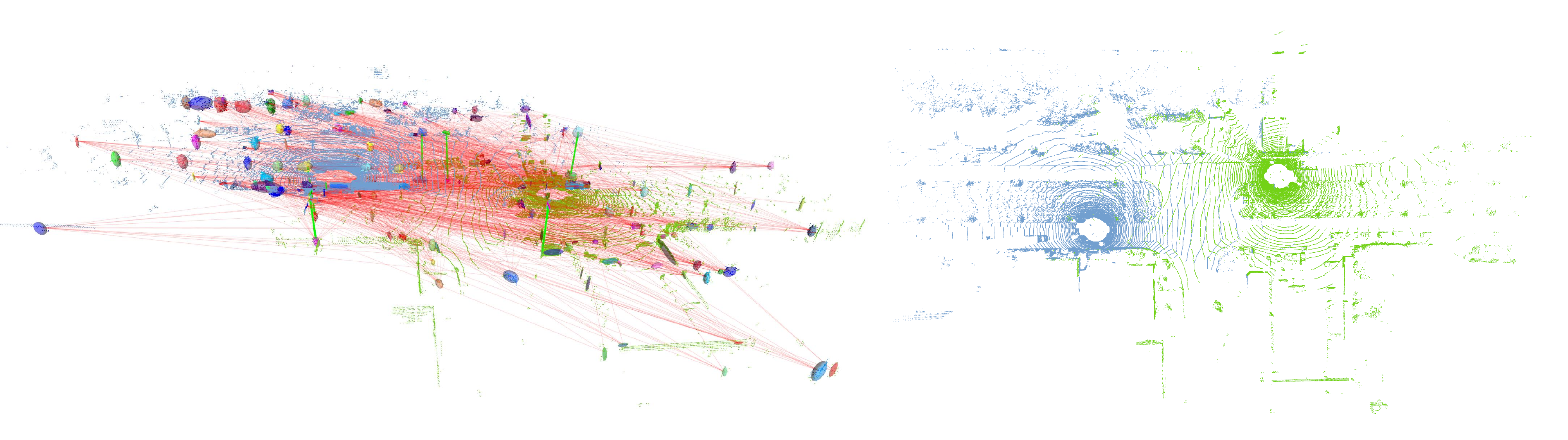}}
	\subfigure[KITTI360-LC (0.75\%)]{	
		\includegraphics[width=0.45\textwidth]{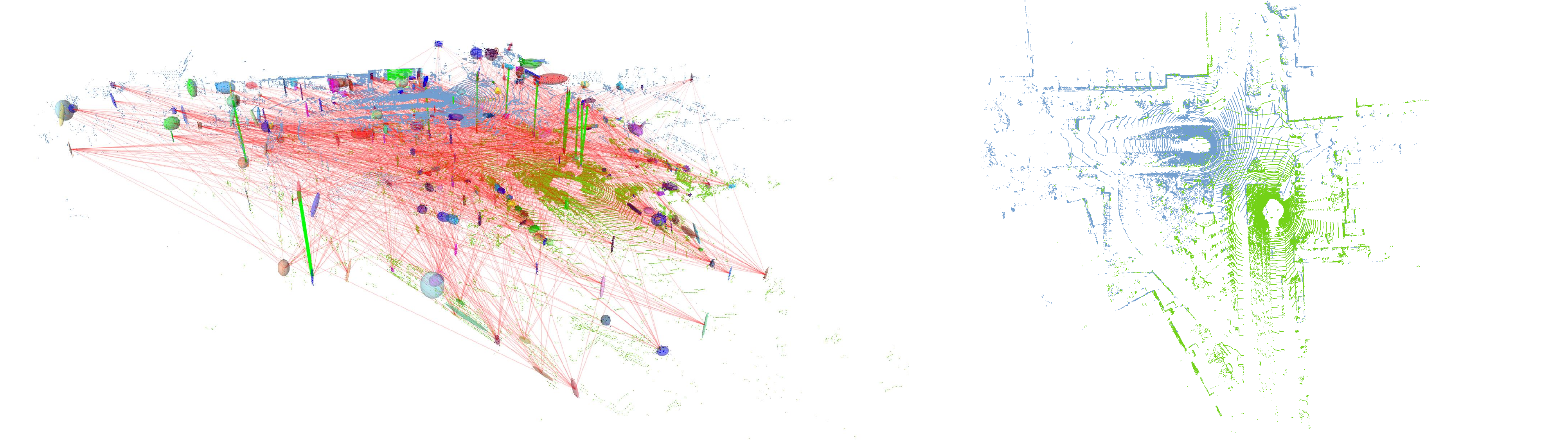}}
        \subfigure[KITTI-LC (0.69\%)]{	
		\includegraphics[width=0.45\textwidth]{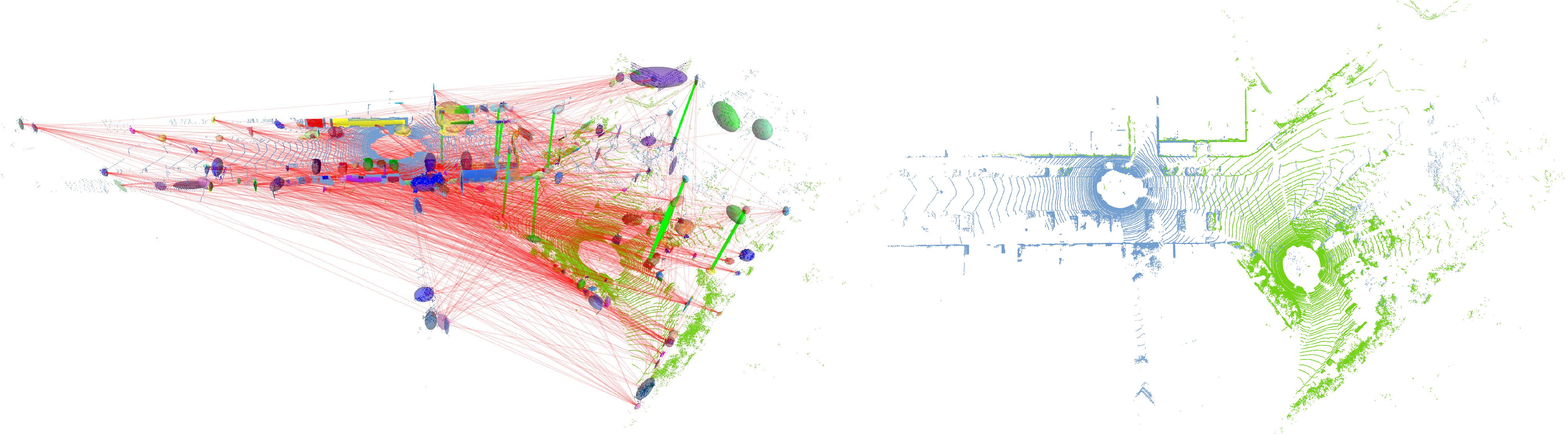}}
	\subfigure[KITTI-LC (1.93\%)]{	
		\includegraphics[width=0.45\textwidth]{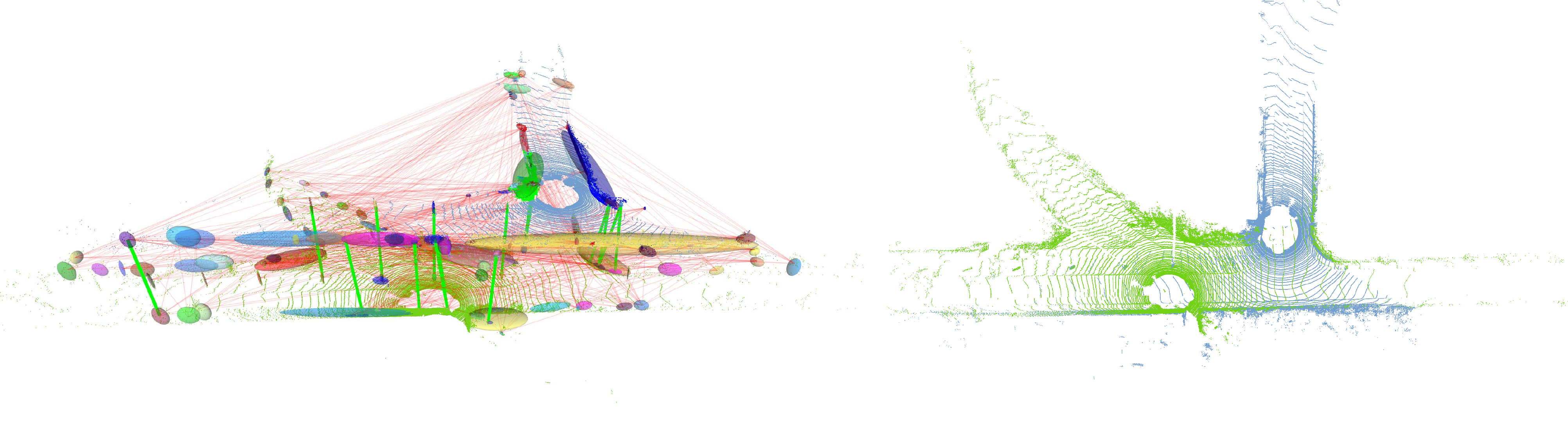}}
        \subfigure[Apollo-LC (0.77\%)]{	
		\includegraphics[width=0.45\textwidth]{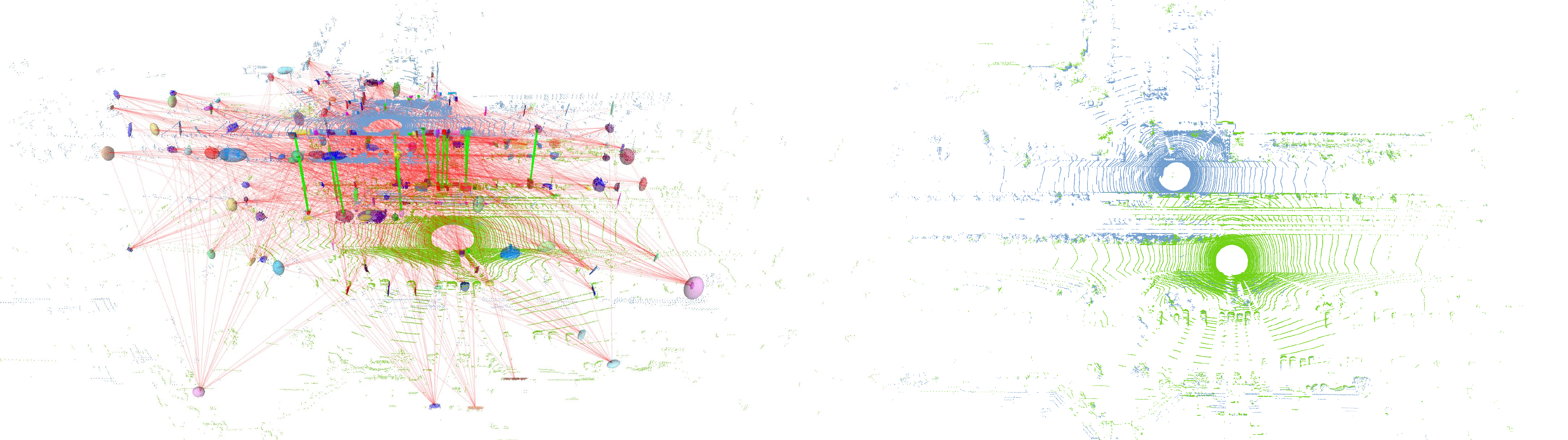}}
	\subfigure[Apollo-LC (1.48\%)]{	
		\includegraphics[width=0.45\textwidth]{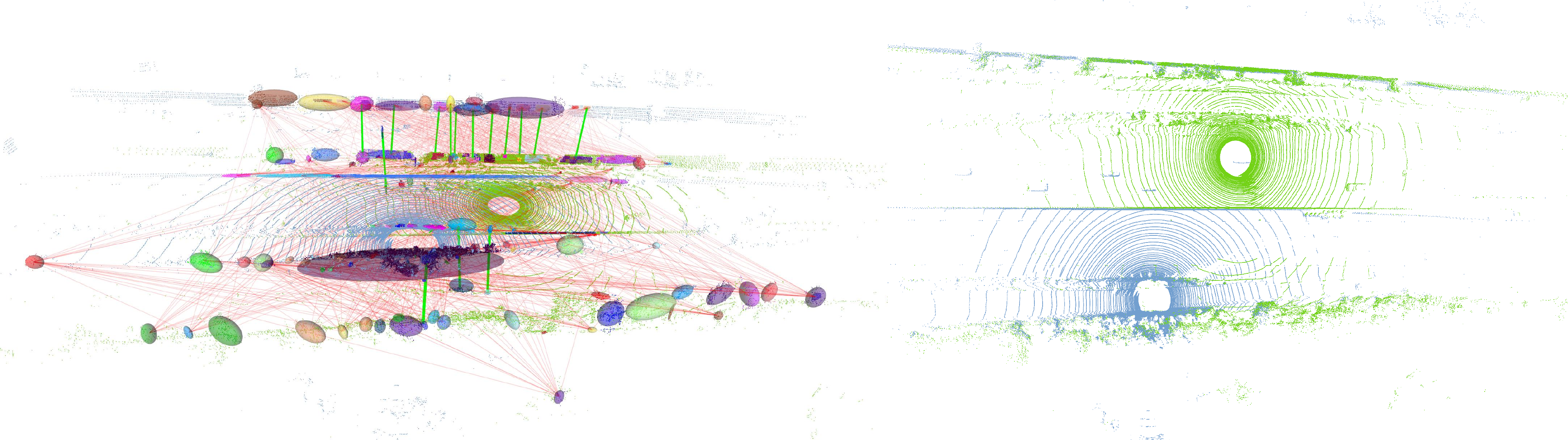}}
        \subfigure[Campus-MS (1.59\%)]{	
		\includegraphics[width=0.45\textwidth]{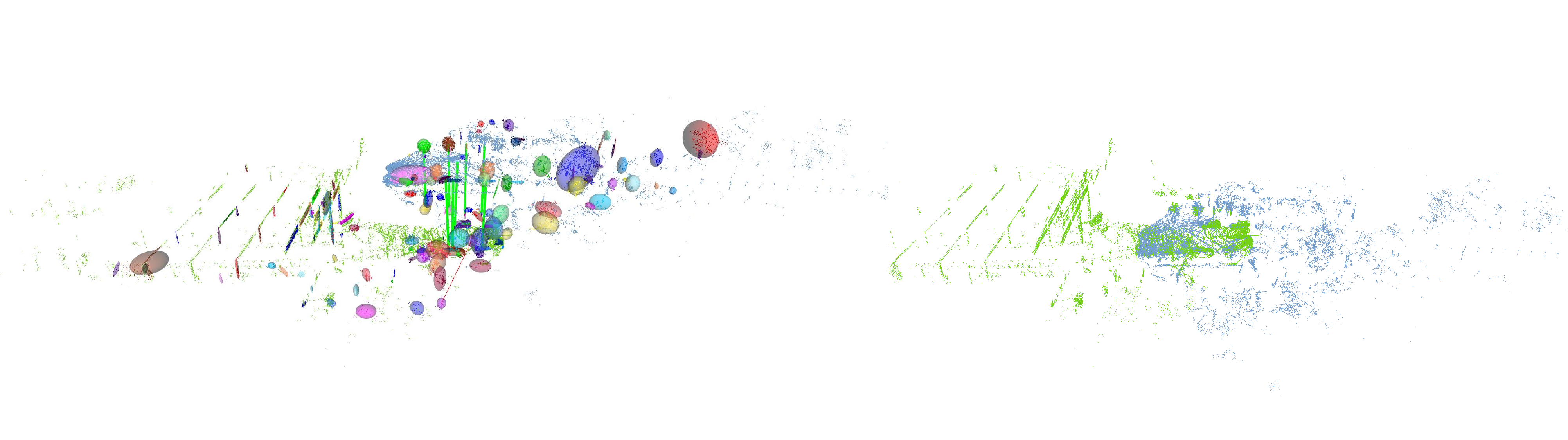}}
	\subfigure[Campus-MS (1.11\%)]{	
		\includegraphics[width=0.45\textwidth]{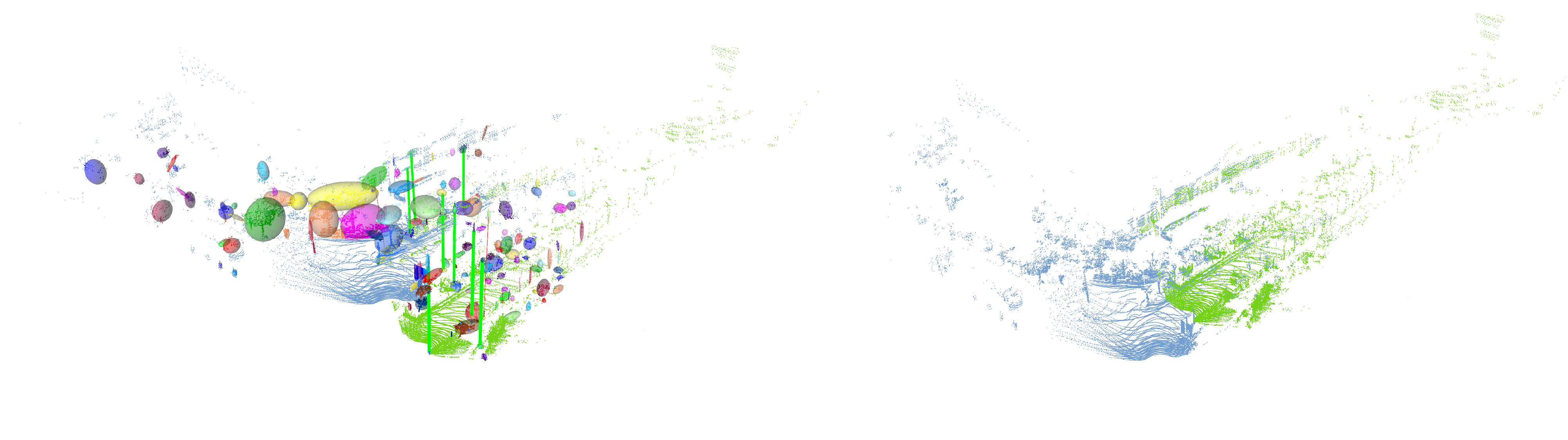}}
	\caption{Challenging cases for global registration where G3Reg could achieve successful alignment even when point clouds exhibited low inlier ratios (shown in brackets). (a)-(f) were tested with Velodyne data and (g) and (h) show registration results using Livox.}
    \label{figure:visualization}
\end{figure*}
\begin{figure}[h]
	\centering
	\subfigure{	
		\includegraphics[width=0.9\columnwidth]{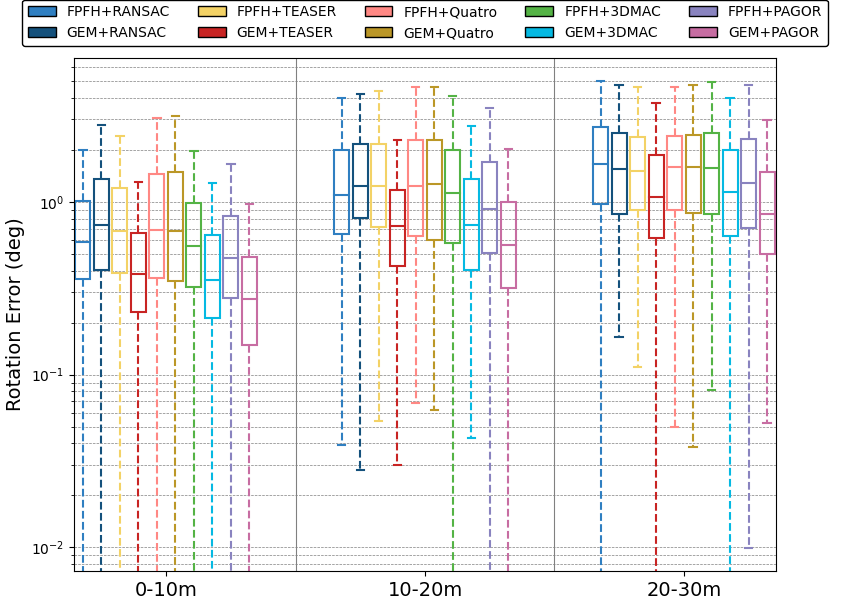}}
	\subfigure{	
		\includegraphics[width=0.9\columnwidth]{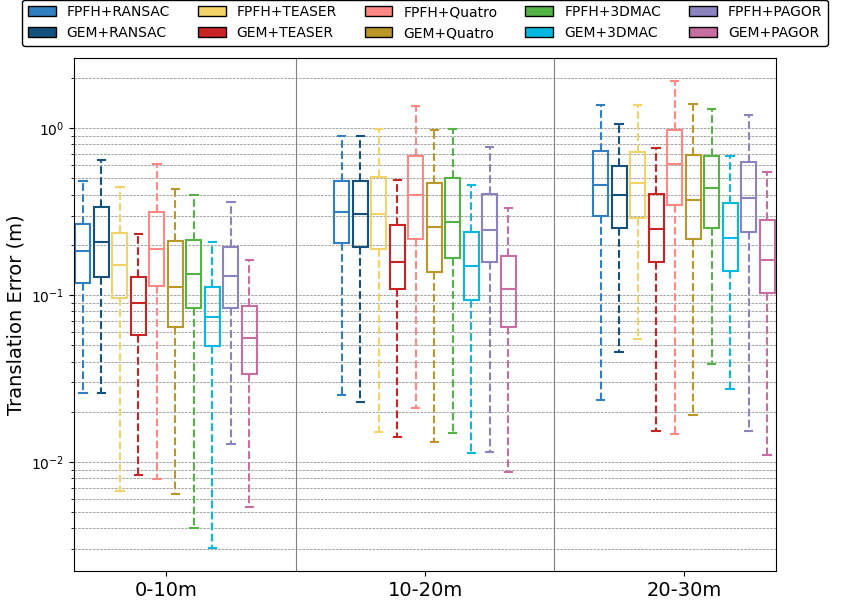}}
	\caption{\revised{Registration accuracy results. Boxplots of rotation errors and translation errors for the ten compared methods on the KITTI-LC dataset.}}
    \label{figure:tf_err}
\end{figure} 

\subsection{Registration Evaluation}
\label{exp:registration}

\begin{table*}[]
\centering
\caption{\revised{Global Registration Evaluation on KITTI-LC Dataset}}
\label{tab:kitti_bm}
\resizebox{\textwidth}{!}{%
\begin{tabular}{c|c|c|ccccccccccccccc|c}
\hline
\multirow{3}{*}{} & \multirow{3}{*}{Front-end} & \multirow{3}{*}{Back-end} & \multicolumn{15}{c|}{Recall (\%) on KITTI-LC $\uparrow$} & \multirow{3}{*}{\begin{tabular}[c]{@{}c@{}}Average\\ Time (ms)\end{tabular}} \\ \cline{4-18}
 &  &  & \multicolumn{5}{c|}{0-10m} & \multicolumn{5}{c|}{10-20m} & \multicolumn{5}{c|}{20-30m} &  \\ \cline{4-18}
 &  &  & 00 & 02 & 05 & 06 & \multicolumn{1}{c|}{08} & 00 & 02 & 05 & 06 & \multicolumn{1}{c|}{08} & 00 & 02 & 05 & 06 & 08 &  \\ \hline
\multicolumn{1}{l|}{\multirow{2}{*}{\begin{tabular}[c]{@{}l@{}}Deep Learning-\\ based Method\end{tabular}}} & FCGF & DGR & 88.12 & 65.12 & 80.11 & 98.45 & \multicolumn{1}{c|}{2.24} & 64.81 & 54.04 & 66.83 & 48.21 & \multicolumn{1}{c|}{0.62} & 41.47 & 35.43 & 45.21 & 27.38 & 0.00 & 546.85 \\
\multicolumn{1}{l|}{} & FCGF & PointDSC & 82.18 & 54.07 & 75.57 & 97.67 & \multicolumn{1}{c|}{2.99} & 61.58 & 46.46 & 62.81 & 47.81 & \multicolumn{1}{c|}{0.00} & 44.88 & 39.46 & 47.49 & 35.32 & 0.00 & 394.17 \\ \hline
\multirow{2}{*}{\begin{tabular}[c]{@{}c@{}}RANSAC-based \\ Method\end{tabular}} & FPFH & RANSAC & {\ul 99.67} & 94.76 & {\ul 99.43} & {\ul 99.22} & \multicolumn{1}{c|}{\textbf{100.0}} & 73.90 & 67.67 & 77.89 & 95.62 & \multicolumn{1}{c|}{91.97} & 25.46 & 29.59 & 32.87 & 78.17 & 58.92 & 230.54 \\
 & GEM & RANSAC & 87.46 & 77.91 & 86.36 & 96.90 & \multicolumn{1}{c|}{88.06} & 35.78 & 43.94 & 40.70 & 60.56 & \multicolumn{1}{c|}{56.79} & 4.72 & 14.80 & 10.96 & 24.60 & 29.19 & 60.91 \\ \hline
\multirow{2}{*}{\begin{tabular}[c]{@{}c@{}}SOTA Loop \\ Closure Method\end{tabular}} & STD & RANSAC & 38.61 & 36.63 & 31.25 & 73.64 & \multicolumn{1}{c|}{35.82} & 0.59 & 1.51 & 0.50 & 19.12 & \multicolumn{1}{c|}{4.93} & 0.00 & 0.00 & 0.00 & 4.76 & 0.54 & 46.70 \\
 & Cont2 & GMM & 91.75 & 80.81 & 96.02 & 96.90 & \multicolumn{1}{c|}{92.54} & 53.66 & 40.40 & 62.81 & 68.92 & \multicolumn{1}{c|}{61.73} & 13.38 & 8.07 & 19.63 & 38.89 & 29.19 & \textbf{5.15} \\ \hline
\multirow{8}{*}{\begin{tabular}[c]{@{}c@{}}Graph-theoratic\\ Method\end{tabular}} & FPFH & TEASER++ & 99.01 & 97.67 & 98.86 & \textbf{100.0} & \multicolumn{1}{c|}{\textbf{100.0}} & 78.59 & 69.70 & 83.42 & {\ul 99.60} & \multicolumn{1}{c|}{95.68} & 30.71 & 30.04 & 37.44 & 87.70 & 56.76 & 207.34 \\
 & FPFH & Quatro & 98.68 & {\ul 99.42} & \textbf{100.0} & \textbf{100.0} & \multicolumn{1}{c|}{{\ul 99.25}} & 89.74 & {\ul 75.25} & 90.95 & 99.20 & \multicolumn{1}{c|}{94.44} & 44.62 & 33.91 & 50.68 & 92.06 & 61.62 & 207.24 \\
 & FPFH & 3DMAC & 99.01 & 93.60 & 98.30 & {\ul 99.22} & \multicolumn{1}{c|}{{\ul 99.25}} & 66.57 & 50.51 & 65.83 & 88.84 & \multicolumn{1}{c|}{84.57} & 14.70 & 19.73 & 21.92 & 63.89 & 41.08 & 203.51 \\
 & FPFH & PAGOR & \textbf{100.0} & \textbf{100.0} & \textbf{100.0} & \textbf{100.0} & \multicolumn{1}{c|}{\textbf{100.0}} & 92.67 & \textbf{86.36} & 90.95 & \textbf{100.0} & \multicolumn{1}{c|}{\textbf{99.38}} & 46.46 & 43.95 & 53.42 & 96.03 & {\ul 73.51} & 283.72 \\
 & GEM & TEASER++ & \textbf{100.0} & 86.05 & \textbf{100.0} & \textbf{100.0} & \multicolumn{1}{c|}{\textbf{100.0}} & 92.67 & 60.10 & 92.46 & \textbf{100.0} & \multicolumn{1}{c|}{96.91} & 44.09 & 44.39 & 61.64 & {\ul 97.22} & 67.03 & 43.67 \\
 & GEM & Quatro & 99.01 & 87.21 & \textbf{100.0} & \textbf{100.0} & \multicolumn{1}{c|}{\textbf{100.0}} & {\ul 93.84} & 62.12 & {\ul 95.48} & 99.20 & \multicolumn{1}{c|}{96.91} & {\ul 47.24} & {\ul 46.19} & {\ul 65.30} & 95.63 & 70.81 & {\ul 42.25} \\
 & GEM & 3DMAC & \textbf{100.0} & 74.42 & \textbf{100.0} & \textbf{100.0} & \multicolumn{1}{c|}{98.51} & 86.51 & 56.57 & 87.44 & 99.20 & \multicolumn{1}{c|}{88.27} & 38.58 & 34.53 & 48.40 & 91.67 & 60.54 & 60.10 \\
 & GEM & PAGOR & \textbf{100.0} & \multicolumn{1}{l}{91.28} & \textbf{100.0} & \textbf{100.0} & \multicolumn{1}{c|}{\textbf{100.0}} & \multicolumn{1}{l}{\textbf{96.77}} & 70.20 & \multicolumn{1}{l}{\textbf{98.49}} & \textbf{100.0} & \multicolumn{1}{l|}{{\ul 97.53}} & \multicolumn{1}{l}{\textbf{62.47}} & \multicolumn{1}{l}{\textbf{52.02}} & \multicolumn{1}{l}{\textbf{74.43}} & \textbf{99.60} & \textbf{78.38} & 64.99 \\ \hline
\end{tabular}%
}
\end{table*}
\begin{table*}[]
\centering
\caption{\revised{Global Registration Evaluation on KITTI360-LC Dataset}}
\label{tab:kitti360_bm}
\resizebox{\textwidth}{!}{%
\begin{tabular}{c|c|c|cccccccccccccccccc|c}
\hline
\multirow{3}{*}{} & \multirow{3}{*}{Front-end} & \multirow{3}{*}{Back-end} & \multicolumn{18}{c|}{Recall (\%) on KITTI360-LC $\uparrow$} & \multirow{3}{*}{\begin{tabular}[c]{@{}c@{}}Average\\ Time (ms)\end{tabular}} \\ \cline{4-21}
 &  &  & \multicolumn{6}{c|}{0-10m} & \multicolumn{6}{c|}{10-20m} & \multicolumn{6}{c|}{20-30m} &  \\ \cline{4-21}
 &  &  & 00 & 02 & 04 & 05 & 06 & \multicolumn{1}{c|}{09} & 00 & 02 & 04 & 05 & 06 & \multicolumn{1}{c|}{09} & 00 & 02 & 04 & 05 & 06 & 09 &  \\ \hline
\multirow{2}{*}{\begin{tabular}[c]{@{}c@{}}Deep Learning-\\ based Method\end{tabular}} & FCGF & DGR & 40.05 & 26.32 & 10.93 & 16.28 & 43.43 & \multicolumn{1}{c|}{55.21} & 24.70 & 19.19 & 9.37 & 14.05 & 32.55 & \multicolumn{1}{c|}{42.28} & 16.83 & 14.23 & 8.10 & 10.68 & 19.54 & 22.24 & 437.15 \\
 & FCGF & PointDSC & 40.58 & 24.56 & 9.90 & 14.95 & 40.83 & \multicolumn{1}{c|}{54.37} & 23.67 & 16.91 & 9.10 & 13.14 & 29.04 & \multicolumn{1}{c|}{39.63} & 16.91 & 13.17 & 7.86 & 10.96 & 19.01 & 25.92 & 335.86 \\ \hline
\multirow{2}{*}{\begin{tabular}[c]{@{}c@{}}RANSAC-based \\ Method\end{tabular}} & FPFH & RANSAC & 98.82 & 98.25 & 98.08 & 93.19 & 97.28 & \multicolumn{1}{c|}{99.23} & 85.03 & 84.75 & 73.09 & 70.54 & 65.40 & \multicolumn{1}{c|}{75.86} & 48.51 & 43.77 & 30.12 & 35.92 & 30.91 & 28.58 & 203.58 \\
 & GEM & RANSAC & 83.90 & 84.21 & 82.13 & 73.26 & 75.38 & \multicolumn{1}{c|}{83.80} & 41.16 & 40.97 & 32.98 & 22.36 & 27.38 & \multicolumn{1}{c|}{31.71} & 12.79 & 12.37 & 9.17 & 5.55 & 6.39 & 7.00 & 66.38 \\ \hline
\multirow{2}{*}{\begin{tabular}[c]{@{}c@{}}SOTA Loop \\ Closure Method\end{tabular}} & STD & RANSAC & 40.58 & 41.35 & 26.74 & 26.74 & 27.22 & \multicolumn{1}{c|}{31.17} & 4.58 & 4.46 & 1.05 & 1.06 & 0.39 & \multicolumn{1}{c|}{1.41} & 0.74 & 0.18 & 0.24 & 0.00 & 0.00 & 0.05 & 48.56 \\
 & Cont2 & GMM & 86.78 & 81.45 & 78.43 & 84.88 & 85.56 & \multicolumn{1}{c|}{91.71} & 55.38 & 37.03 & 42.48 & 40.18 & 36.35 & \multicolumn{1}{c|}{50.31} & 24.67 & 8.18 & 19.05 & 13.04 & 9.95 & 16.82 & \textbf{12.44} \\ \hline
\multirow{8}{*}{\begin{tabular}[c]{@{}c@{}}Graph-theoratic\\ Method\end{tabular}} & FPFH & TEASER++ & 99.34 & 98.12 & 97.78 & 93.68 & {\ul 97.04} & \multicolumn{1}{c|}{99.29} & 87.18 & 84.23 & 74.80 & 72.66 & 67.64 & \multicolumn{1}{c|}{79.82} & 55.12 & 48.49 & 34.64 & 35.37 & 30.11 & 34.15 & 184.7 \\
 & FPFH & Quatro & 99.61 & 97.99 & {\ul 98.37} & {\ul 97.51} & 96.45 & \multicolumn{1}{c|}{99.81} & 91.86 & 84.65 & 83.51 & 83.53 & 74.95 & \multicolumn{1}{c|}{88.97} & 61.63 & 54.63 & 45.48 & 50.62 & 42.63 & 45.24 & 184.62 \\
 & FPFH & 3DMAC & 98.17 & 98.24 & 96.31 & 89.70 & 95.38 & \multicolumn{1}{c|}{98.32} & 76.80 & 73.03 & 59.23 & 58.00 & 54.29 & \multicolumn{1}{c|}{64.84} & 38.20 & 29.80 & 20.71 & 23.72 & 19.80 & 19.17 & 183.31 \\
 & FPFH & PAGOR & \textbf{100.0} & \textbf{100.0} & \textbf{99.70} & 97.67 & \textbf{99.05} & \multicolumn{1}{c|}{{\ul 99.93}} & \textbf{94.11} & \textbf{93.77} & \textbf{91.82} & \textbf{88.06} & \textbf{84.50} & \multicolumn{1}{c|}{{\ul 91.91}} & {\ul 68.56} & \textbf{65.48} & {\ul 53.21} & \textbf{53.67} & \textbf{49.20} & {\ul 49.64} & 212.02 \\
 & GEM & TEASER++ & {\ul 99.74} & 98.12 & 97.64 & 96.51 & 96.92 & \multicolumn{1}{c|}{99.81} & 89.62 & 80.71 & 75.99 & 69.33 & 67.25 & \multicolumn{1}{c|}{86.26} & 60.89 & 46.97 & 41.75 & 28.07 & 29.04 & 43.97 & 53.63 \\
 & GEM & Quatro & 99.61 & 96.11 & 97.64 & 97.18 & 95.26 & \multicolumn{1}{c|}{99.74} & 89.99 & 80.50 & 77.57 & 73.87 & 67.84 & \multicolumn{1}{c|}{89.03} & 62.71 & 47.86 & 44.05 & 34.40 & 31.79 & 47.95 & {\ul 53.52} \\
 & GEM & 3DMAC & 99.61 & 95.36 & 96.75 & 93.19 & 95.38 & \multicolumn{1}{c|}{98.78} & 84.66 & 66.08 & 64.64 & 48.04 & 56.43 & \multicolumn{1}{c|}{78.35} & 46.29 & 31.05 & 29.76 & 19.00 & 21.76 & 36.45 & 61.38 \\
 & GEM & PAGOR & \textbf{100.0} & {\ul 99.62} & \textbf{99.70} & \textbf{99.50} & \textbf{99.05} & \multicolumn{1}{c|}{\textbf{100.0}} & {\ul 92.89} & {\ul 91.18} & {\ul 86.41} & {\ul 85.35} & {\ul 79.53} & \multicolumn{1}{c|}{\textbf{94.52}} & \textbf{72.19} & {\ul 61.92} & \textbf{58.45} & {\ul 49.51} & {\ul 46.36} & \textbf{62.53} & 76.89 \\ \hline
\end{tabular}%
}
\end{table*}
\begin{table}[]
\centering
\caption{Global Registration Evaluation on Apollo-LC Dataset}
\label{tab:apollo_bm}
\resizebox{\columnwidth}{!}{%
\begin{tabular}{c|c|cccc|c}
\hline
\multirow{2}{*}{Front-end} & \multirow{2}{*}{Back-end} & \multicolumn{4}{c|}{Recall (\%) on APOLLO-LC $\uparrow$} & \multirow{2}{*}{\begin{tabular}[c]{@{}c@{}}Average\\ Time (ms)\end{tabular}} \\ \cline{3-6}
 &  & 0-10m & 10-20m & 20-30m & 30-40m &  \\ \hline
FPFH & RANSAC & 99.69 & 92.65 & 67.86 & 35.92 & 272.73 \\
FPFH & TEASER++ & 99.85 & 95.00 & 75.23 & 44.53 & 251.62 \\
FPFH & Quatro & 99.68 & 94.17 & 76.10 & 48.05 & 255.39 \\
FPFH & 3DMAC & 99.09 & 84.32 & 51.71 & 23.61 & 257.92 \\
FPFH & PAGOR & 99.99 & 97.78 & 86.04 & 60.67 & 301.86 \\ \hline
GEM & RANSAC & 97.46 & 74.86 & 42.01 & 19.91 & 52.82 \\
GEM & TEASER++ & 99.98 & {\ul 98.30} & {\ul 91.89} & 77.39 & \textbf{38.56} \\
GEM & Quatro & {\ul 99.99} & 98.17 & 91.71 & {\ul 77.91} & {\ul 38.63} \\
GEM & 3DMAC & 99.80 & 95.56 & 84.91 & 68.33 & 55.58 \\
GEM & PAGOR & \textbf{100.0} & \textbf{99.15} & \textbf{95.50} & \textbf{84.51} & 64.31 \\ \hline
\end{tabular}%
}
\end{table}
\subsubsection{KITTI-10m Dataset}

\revised{Although the KITTI-10m dataset has limitations in registration performance evaluation, we include it in Table \ref{tab:kitti_10m} for consistency with prior studies. We use the more accurate Semantic KITTI\cite{behley2021ijrr} ground truth for benchmarking, addressing inaccuracies in the official KITTI ground truth. Table \ref{tab:kitti_10m} shows that deep learning-based methods using new ground truth pose perform comparably to their original reported results. By adjusting FPFH and TEASER++ parameters, we found traditional methods outperform those in earlier research\cite{zhang20233d}. It should be noted that a discrepancy in 3DMAC performance arises when using MNN for correspondence filtering, resulting in fewer correspondences (5000 in the original paper\cite{zhang20233d}) and decreased results.}

\subsubsection{Results on Loop Closure Datasets}

\revised{In this study, we assess our registration method on loop closure point clouds with distances up to $30\sim40$ meters, meeting most robotics application requirements. Results on KITTI-LC, KITT360-LC, and Apollo-LC datasets are detailed in Tables \ref{tab:kitti_bm}, \ref{tab:kitti360_bm}, and \ref{tab:apollo_bm}, respectively. Methods are categorized into deep learning-based, RANSAC-based, and graph-theoretical at the backend, with our GEM and PAGOR performances evaluated when integrated with these methods. We also benchmark against state-of-the-art loop closure techniques capable of global registration. Additional GEM matching and registration results are visualized in Figure \ref{figure:visualization}.}

\revised{Deep learning methods show a performance drop in loop closure scenarios, as Tables \ref{tab:kitti_bm} and \ref{tab:kitti360_bm} depict. This is attributed to the challenging nature of loop closures with varied viewpoints and a training set on consecutive frames with limited scale and pose diversity. RANSAC-based methods perform well with low outlier ratios when using high-quality descriptors like FPFH within a 10m translation. Yet, their performance declines with larger translations, or when using GEM with high outlier ratios, as the computational complexity of RANSAC increases exponentially with the outlier ratio\cite{inproceedings}. Our PAGOR outshines TEASER++ and Quatro in graph-theoretical methods, benefiting from our distrust-and-verify framework, with only a marginal 20-millisecond time increase. Section \ref{exp:ab_study} explores this performance-efficiency trade-off in depth.}

\revised{3DMAC shares our hypothesis generation and evaluation philosophy. Admittedly, the clique consisting of inliers must be the maximal clique as depicted in TEASER\cite{yang2020teaser}. However, it may miss the true maximal clique due to SVD inaccuracies and an unfit evaluation function, as discussed in Section \ref{sec:verify}. In contrast to 3DMAC, our approach focuses on using multiple compatibility thresholds and an accurate evaluation function to select a maximum clique that is the true maximal clique.}

In addition to specialized registration methods, we also test advanced loop closure systems with metric pose estimation (STD\cite{yuan2023std}, Cont2\cite{jiang2023contour}) on the same datasets. We find that STD is sensitive to the physical distance of two point clouds, which we attribute to potential shortcomings in the reproducibility of their feature points (despite increasing the number to 1000). Cont2 performs fairly well with point clouds close to each other, and the speed is noticeably faster than all others, owing to its remarkably efficient scan representation. We note that both methods typically balance recall and precision of loop detection, therefore they could be too conservative for registration-only tasks.

\subsubsection{Registration Accuracy}

\revised{As shown in Figure~\ref{figure:tf_err}, we compared the performance of GEM and FPFH when integrated with various back-end solutions, and we applied the same to PAGOR. Consistent with the registration recall benchmark, both GEM and PAGOR outperform the other methods in terms of accuracy. This superior performance can be attributed to two main factors: First, unlike point-based registration methods, GEM-based registration leverages additional geometric constraints including plane-to-plane and line-to-line correspondences, followed by a distribution-to-distribution registration approach. Second, during the verification stage, the evaluation function preferentially selects the most accurate transformation, further enhancing performance.}

\subsection{Ablation Study}
\label{exp:ab_study}

\begin{table*}[]
\centering
\caption{\revised{Ablation Study on KITTI-LC Dataset}}
\label{tab:kitti_ab}
\resizebox{\textwidth}{!}{%
\begin{tabular}{cc|c|clcccllllllllcc|c}
\hline
\multicolumn{2}{c|}{} &  & \multicolumn{15}{c|}{Recall Rate (\%) on KITTI-LC $\uparrow$} &  \\ \cline{4-18}
\multicolumn{2}{c|}{} &  & \multicolumn{5}{c|}{0-10m} & \multicolumn{5}{c|}{10-20m} & \multicolumn{5}{c|}{20-30m} &  \\ \cline{4-18}
\multicolumn{2}{c|}{\multirow{-3}{*}{Variables}} & \multirow{-3}{*}{Values} & 00 & \multicolumn{1}{c}{02} & 05 & 06 & \multicolumn{1}{c|}{08} & \multicolumn{1}{c}{00} & \multicolumn{1}{c}{02} & \multicolumn{1}{c}{05} & \multicolumn{1}{c}{06} & \multicolumn{1}{c|}{08} & \multicolumn{1}{c}{00} & \multicolumn{1}{c}{02} & \multicolumn{1}{c}{05} & 06 & 08 & \multirow{-3}{*}{Average Time (ms)} \\ \hline
Plane-aided Seg & (1) & w/o & \textbf{100.0} & \multicolumn{1}{c}{\textbf{95.35}} & \textbf{100.0} & \textbf{100.0} & \multicolumn{1}{c|}{\textbf{100.0}} & \multicolumn{1}{c}{89.15} & \multicolumn{1}{c}{70.20} & \multicolumn{1}{c}{92.46} & \multicolumn{1}{c}{99.60} & \multicolumn{1}{c|}{94.44} & \multicolumn{1}{c}{40.16} & \multicolumn{1}{c}{46.64} & \multicolumn{1}{c}{63.01} & 95.24 & 70.27 & 57.63 \\ \hline
 & (2) & Plane & 92.74 & \multicolumn{1}{c}{50.58} & 87.50 & 99.22 & \multicolumn{1}{c|}{75.37} & \multicolumn{1}{c}{48.68} & \multicolumn{1}{c}{26.26} & \multicolumn{1}{c}{47.24} & \multicolumn{1}{c}{85.66} & \multicolumn{1}{c|}{37.65} & \multicolumn{1}{c}{15.22} & \multicolumn{1}{c}{8.97} & \multicolumn{1}{c}{17.35} & 59.52 & 8.65 & 44.98 \\
 & (3) & Cluster & {\ul 99.67} & \multicolumn{1}{c}{93.60} & 98.86 & \textbf{100.0} & \multicolumn{1}{c|}{98.51} & \multicolumn{1}{c}{73.90} & \multicolumn{1}{c}{55.55} & \multicolumn{1}{c}{80.90} & \multicolumn{1}{c}{98.41} & \multicolumn{1}{c|}{84.57} & \multicolumn{1}{c}{26.77} & \multicolumn{1}{c}{23.77} & \multicolumn{1}{c}{42.92} & 85.32 & 44.86 & 43.72 \\
 & (4) & Line & 66.34 & \multicolumn{1}{c}{41.86} & 50.00 & 85.27 & \multicolumn{1}{c|}{74.63} & \multicolumn{1}{c}{30.79} & \multicolumn{1}{c}{26.26} & \multicolumn{1}{c}{26.63} & \multicolumn{1}{c}{47.41} & \multicolumn{1}{c|}{62.96} & \multicolumn{1}{c}{9.71} & \multicolumn{1}{c}{15.25} & \multicolumn{1}{c}{10.96} & 23.01 & 43.78 & 37.41 \\
 & (5) & Plane+Cluster & \textbf{100.0} & \multicolumn{1}{c}{90.12} & \textbf{100.0} & \textbf{100.0} & \multicolumn{1}{c|}{\textbf{100.0}} & \multicolumn{1}{c}{94.13} & \multicolumn{1}{c}{67.68} & \multicolumn{1}{c}{91.46} & \multicolumn{1}{c}{100.0} & \multicolumn{1}{c|}{95.06} & \multicolumn{1}{c}{44.62} & \multicolumn{1}{c}{34.98} & \multicolumn{1}{c}{58.45} & 97.22 & 65.41 & 63.54 \\
 & (6) & Plane+Line & 99.34 & \multicolumn{1}{c}{65.70} & 94.32 & \textbf{100.0} & \multicolumn{1}{c|}{99.25} & \multicolumn{1}{c}{78.01} & \multicolumn{1}{c}{57.07} & \multicolumn{1}{c}{73.37} & \multicolumn{1}{c}{97.61} & \multicolumn{1}{c|}{82.72} & \multicolumn{1}{c}{31.76} & \multicolumn{1}{c}{39.46} & \multicolumn{1}{c}{38.81} & 79.76 & 57.84 & 47.62 \\
\multirow{-6}{*}{GEM Type} & (7) & Line+Cluster & \textbf{100.0} & \multicolumn{1}{c}{90.70} & \textbf{100.0} & \textbf{100.0} & \multicolumn{1}{c|}{\textbf{100.0}} & \multicolumn{1}{c}{90.91} & \multicolumn{1}{c}{63.13} & \multicolumn{1}{c}{91.96} & \multicolumn{1}{c}{99.60} & \multicolumn{1}{c|}{95.68} & \multicolumn{1}{c}{45.67} & \multicolumn{1}{c}{44.39} & \multicolumn{1}{c}{61.19} & 90.87 & 71.89 & 49.56 \\ \hline
 & (8) & [0.99] & \textbf{100.0} & \multicolumn{1}{c}{79.65} & \textbf{100.0} & \textbf{100.0} & \multicolumn{1}{c|}{\textbf{100.0}} & \multicolumn{1}{c}{93.55} & \multicolumn{1}{c}{59.60} & \multicolumn{1}{c}{92.96} & \multicolumn{1}{c}{\textbf{100.0}} & \multicolumn{1}{c|}{96.30} & \multicolumn{1}{c}{42.78} & \multicolumn{1}{c}{42.15} & \multicolumn{1}{c}{54.34} & 94.05 & 64.86 & 46.73 \\
 & (9) & [0.95] & \textbf{100.0} & 83.72 & \textbf{100.0} & \textbf{100.0} & \multicolumn{1}{c|}{\textbf{100.0}} & 94.43 & 63.13 & 94.97 & \textbf{100.0} & \multicolumn{1}{l|}{96.91} & 46.98 & 46.19 & 64.38 & 98.02 & 70.27 & 47.18 \\
 & (10) & [0.9] & \textbf{100.0} & 86.05 & \textbf{100.0} & \textbf{100.0} & \multicolumn{1}{c|}{\textbf{100.0}} & 94.13 & 63.64 & 94.97 & \textbf{100.0} & \multicolumn{1}{l|}{\textbf{98.15}} & 49.87 & 45.74 & 65.75 & 97.22 & 69.19 & 48.82 \\
 & (11) & [0.8] & \textbf{100.0} & 86.04 & \textbf{100.0} & \textbf{100.0} & \multicolumn{1}{c|}{\textbf{100.0}} & 93.55 & 65.66 & 96.48 & \textbf{100.0} & \multicolumn{1}{l|}{96.91} & 46.72 & 43.95 & 64.38 & 98.41 & 69.73 & 49.02 \\
 & (12) & [0.99, 0.95] & \textbf{100.0} & 86.05 & \textbf{100.0} & \textbf{100.0} & \multicolumn{1}{c|}{\textbf{100.0}} & 96.19 & 64.65 & 96.48 & \textbf{100.0} & \multicolumn{1}{l|}{{\ul 97.53}} & 53.28 & 48.43 & 68.03 & 98.02 & 71.35 & 49.33 \\
\multirow{-6}{*}{p-values} & (13) & [0.99, 0.95, 0.9] & \textbf{100.0} & 88.95 & \textbf{100.0} & \textbf{100.0} & \multicolumn{1}{c|}{\textbf{100.0}} & 96.48 & 66.16 & 97.49 & \textbf{100.0} & \multicolumn{1}{l|}{\textbf{98.15}} & 59.06 & 50.22 & 73.06 & 98.41 & 74.59 & 55.37 \\ \hline
Comp. Thd & (14) & [0.2m,0.4m,0.6m,0.8m] & \textbf{100.0} & {\ul 91.86} & \textbf{100.0} & \textbf{100.0} & \multicolumn{1}{c|}{\textbf{100.0}} & 96.77 & \textbf{72.22} & 96.98 & \textbf{100.0} & \multicolumn{1}{l|}{\textbf{98.15}} & 57.22 & \textbf{54.26} & 70.78 & {\ul 99.21} & 74.59 & {\color[HTML]{FD6864} 70.99} \\ \hline
Covariance & (15) & $\Sigma$ & \textbf{100.0} & {\ul 91.86} & \textbf{100.0} & \textbf{100.0} & \multicolumn{1}{c|}{\textbf{100.0}} & 96.19 & 66.16 & {\ul 97.99} & \textbf{100.0} & \multicolumn{1}{l|}{\textbf{98.15}} & 57.48 & {\ul 52.47} & 73.97 & {\ul 99.21} & 76.76 & {\color[HTML]{FD6864} 71.79} \\ \hline
Center & (16) & $\overline{\mu}$ & \textbf{100.0} & 86.05 & \textbf{100.0} & \textbf{100.0} & \multicolumn{1}{c|}{\textbf{100.0}} & {\ul 97.36} & 68.18 & 95.48 & \textbf{100.0} & \multicolumn{1}{l|}{\textbf{98.15}} & \textbf{65.62} & 52.02 & 73.52 & 98.81 & \textbf{82.16} & {\color[HTML]{333333} 64.22} \\ \hline
 & (17) & Tukey & \textbf{100.0} & 90.70 & \textbf{100.0} & \textbf{100.0} & \multicolumn{1}{c|}{\textbf{100.0}} & \textbf{97.65} & 69.19 & \textbf{98.49} & \textbf{100.0} & \multicolumn{1}{l|}{{\ul 97.53}} & 61.94 & 52.02 & {\ul 74.43} & \textbf{99.60} & {\ul 77.84} & 63.09 \\
 & (18) & Cauchy & \textbf{100.0} & 89.53 & \textbf{100.0} & \textbf{100.0} & \multicolumn{1}{c|}{\textbf{100.0}} & 96.77 & 68.18 & 95.48 & \textbf{100.0} & \multicolumn{1}{l|}{\textbf{98.15}} & 54.33 & 50.22 & 67.58 & {\ul 99.21} & 74.05 & 64.90 \\
 & (19) & TLS & \textbf{100.0} & 91.28 & \textbf{100.0} & \textbf{100.0} & \multicolumn{1}{c|}{\textbf{100.0}} & {\ul 97.36} & 69.19 & \textbf{98.49} & \textbf{100.0} & \multicolumn{1}{l|}{\textbf{98.15}} & 59.84 & 51.57 & \textbf{74.89} & {\ul 99.21} & 76.76 & 63.12 \\
\multirow{-4}{*}{\begin{tabular}[c]{@{}c@{}}Robust \\ Kernal \\ Function\end{tabular}} & (20) & Huber & \textbf{100.0} & 89.53 & \multicolumn{1}{l}{\textbf{100.0}} & \textbf{100.0} & \multicolumn{1}{c|}{\textbf{100.0}} & 95.01 & 67.17 & 93.47 & \textbf{100.0} & \multicolumn{1}{l|}{{\ul 97.53}} & 47.24 & 47.53 & 65.30 & 98.02 & 70.81 & 64.48 \\ \hline
Default & (21) & See Table \ref{tab:paras} & \textbf{100.0} & 91.28 & \textbf{100.0} & \textbf{100.0} & \multicolumn{1}{c|}{\textbf{100.0}} & \textbf{96.77} & \multicolumn{1}{c}{{\ul 70.20}} & \textbf{98.49} & \textbf{100.0} & \multicolumn{1}{l|}{{\ul 97.53}} & {\ul 62.47} & 52.02 & {\ul 74.43} & \textbf{99.60} & 78.38 & 64.99 \\ \hline
\end{tabular}%
}
\end{table*}
\begin{figure}[t]
	\centering
	\subfigure{	
		\includegraphics[width=0.9\columnwidth]{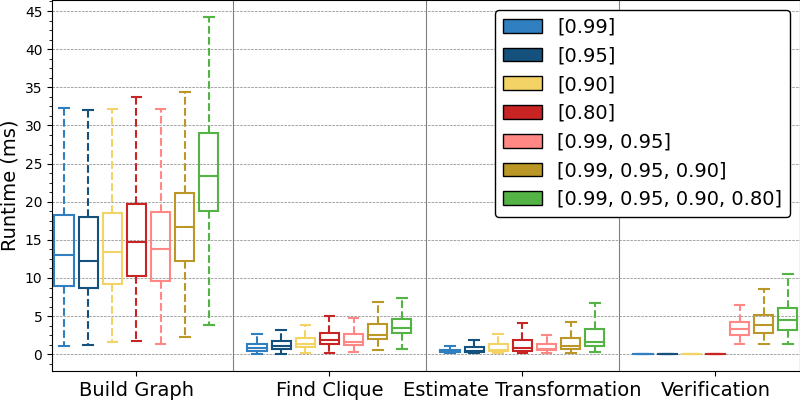}}
	\caption{\revised{The time consumption of PAGOR with different p-values combinations.}}
    \label{figure:runtime}
\end{figure}

In this section, we validate the effectiveness of the main contributions of our method via conducting ablation studies. \revised{All results are shown in Table \ref{tab:kitti_ab}. In Row 1, we bypass plane-aided segmentation and apply direct clustering, subsequently categorizing the results into PCL types based on eigenvalues. A comparative analysis to Row 21 reveals that plane-aided segmentation's benefits increase with translation magnitude. In translations up to 10 meters, an exception is found in sequence 02, sourced from highways with minimal planar features, suggesting that plane-aided segmentation might lead to over-segmentation and less repeatable clustering for this case. Rows 2-7 show the registration impact of various PCL type combinations, with the general trend indicating that clusters have a more pronounced effect than planes, and planes more than lines. Introducing an additional geometric type consistently improves registration accuracy.}

\revised{In rows 8--13 of the dataset, we begin by considering the smallest p-values and incrementally include larger ones, evaluating them both as individual factors and in various combinations. The results show that combinations of multiple p-values consistently yield higher recall rates than either individual p-values or any smaller subsets thereof. As illustrated in Figure~\ref{figure:runtime}, the time required by PAGOR increases with the addition of each p-value. Notably, constructing the compatibility graph is the most time-consuming process. Moreover, our proposed graduated maximum clique (MAC) solver ensures the time complexity of finding MACs within a pyramid graph does not increase linearly.}

\revised{In Rows 14 and 15, we apply predetermined compatibility thresholds such as 0.2m, 0.4m, 0.6m, and 0.8m, or we substitute our proposed pseudo covariance matrix with a statistical covariance matrix $\Sigma$ to compute thresholds. These modifications result in an increase in computational time. This can be attributed to the potential introduction of inaccurate compatibility thresholds, which introduce unnecessary graph edges, thereby reducing efficiency. In Row 16, we experiment with utilizing the geometric centroid $\overline{\mu}$ of the 3D bounding box instead of the default statistical center $\mu$ when constructing TRIMs for registration. This method shows a slight performance increase of sequence 00 and 08 in the [20m, 30m] range, possibly because the symmetry of the 3D box mitigates the center uncertainty attributable to partial observations. Nevertheless, due to the 'distrust-and-verify' framework employed, performance fluctuations remain minimal.}

\revised{Lastly, in Rows 17-20, as detailed in Section \ref{sec:verify}, we examine a variety of robust kernel functions. The experiments demonstrate that DCS (default), Tukey, and TLS kernels outperform other alternatives in terms of recall rate.}

\subsection{Limitation}
\begin{figure*}[t]
	\centering
	\includegraphics[width=\linewidth]{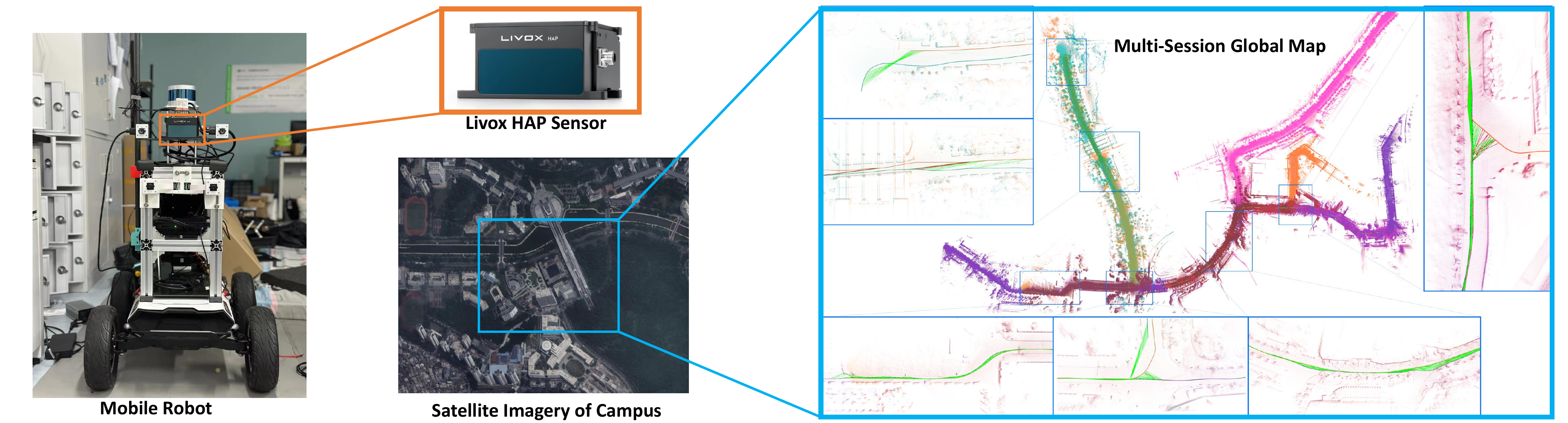}
	\caption{The visualized results of multi-session mapping using our proposed G3Reg. The Campus-MS dataset is collected on a mobile robot platform equipped with a solid-state Livox HAP Sensor. Though the sensor type is different from Velodyne, we still use the same parameter settings for G3Reg. As a result, G3Reg provides stable transformation constraints for the multi-session mapping system (shown in green lines, best viewed in an enlarged format).}
	\label{fig:campus}
\end{figure*}

\begin{figure}[t]
	\centering
	\subfigure[GEM Matching]{	
		\includegraphics[width=8cm]{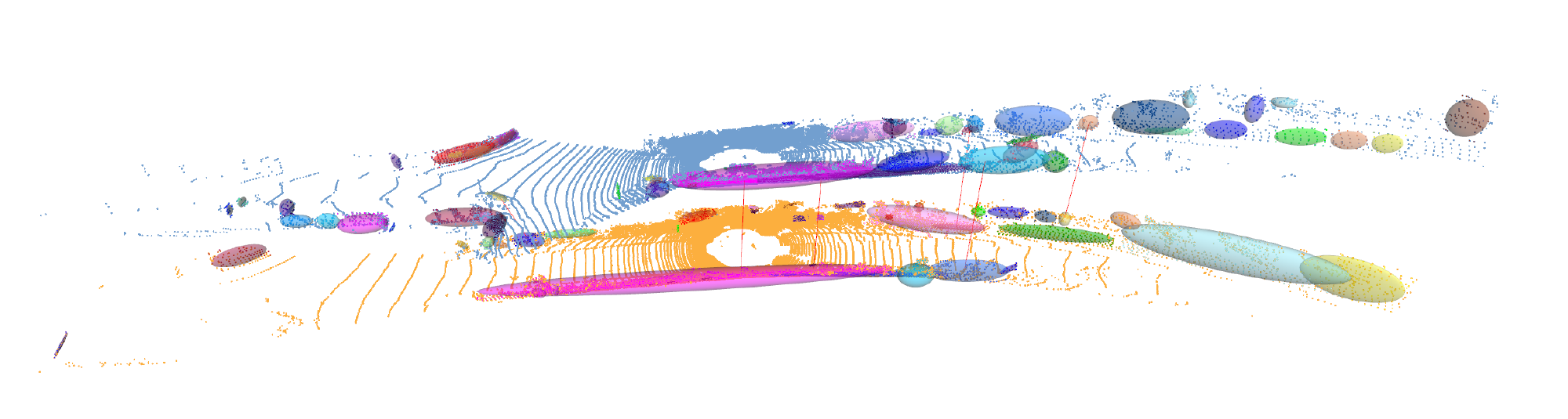}}
	\subfigure[Registration result with G3Reg]{	
		\includegraphics[width=8cm]{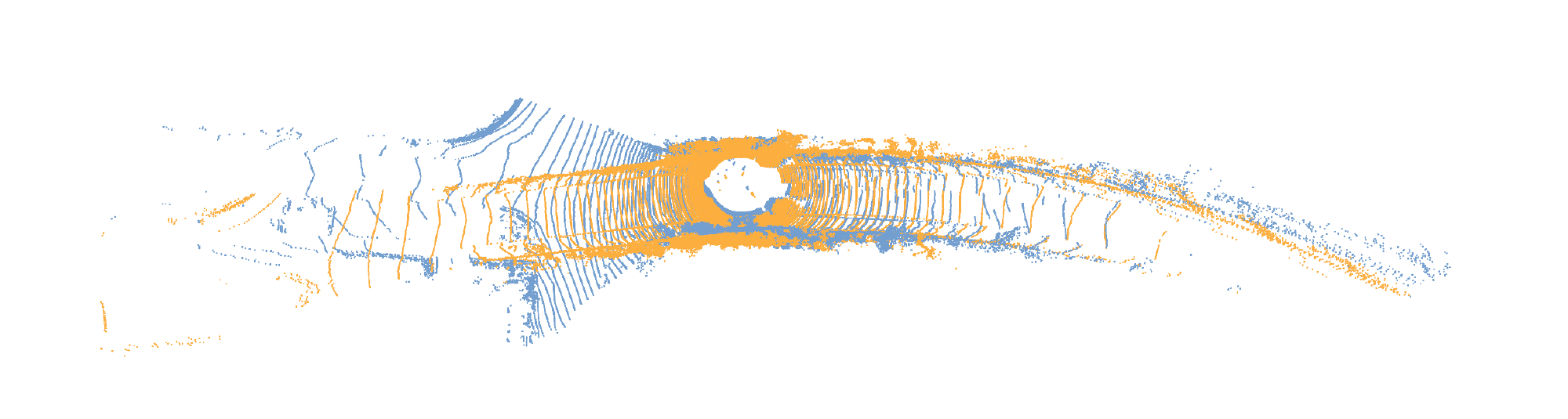}}
    \subfigure[Ground truth result]{	
		\includegraphics[width=8cm]{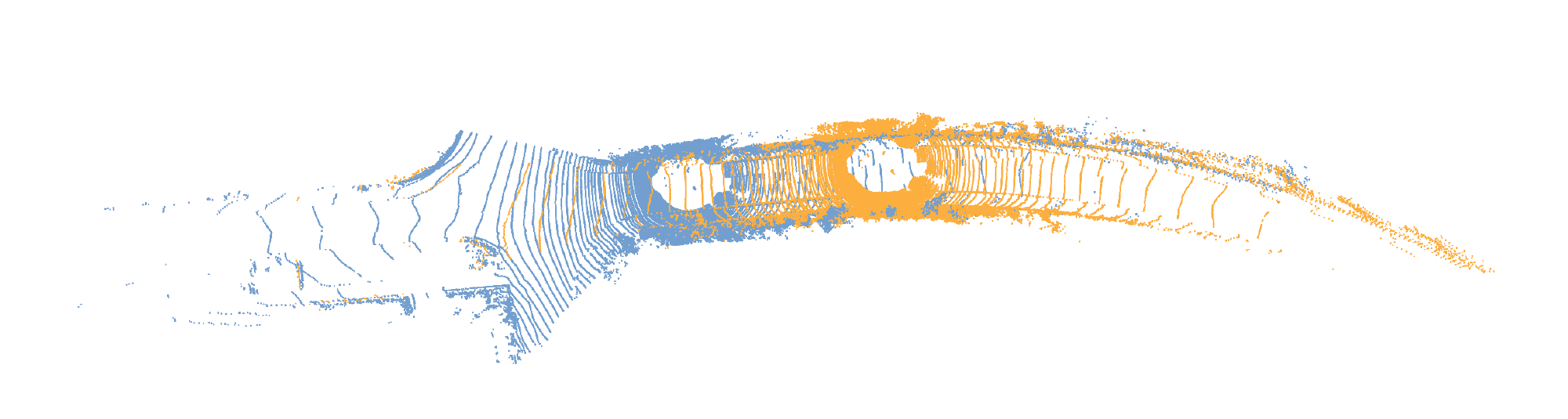}}
	\caption{Global registration failure using G3Reg. This is a very challenging case with a vehicle moving on a narrow, geometrically uninformative road. The ground truth translation is 26.18 m.}
    \label{figure:failure}
\end{figure}

\revised{The primary limitation of the proposed G3Reg lies in its front-end processing. As a segment-based registration method, it necessitates clear segmentation of the scene. This requirement does not imply that the scene must be inherently structured, like a town, but rather that it should be free of extensive, interconnected objects such as continuous rows of bushes along a road or walls flanking a corridor. Such repetitive and expansive features tend to introduce ambiguity in GEM matching and are illustrated in Figure \ref{figure:failure}, where our method struggles in environments that are geometrically uninformative. In these cases, GEMs are prone to significant uncertainties and exhibit low repeatability.}

\revised{To mitigate these limitations, one potential strategy could be to extract FPFH within the point cloud segment of the large GEMs, which could potentially increase the inlier rate and registration recall. Alternatively, leveraging semantic information could enable better scene segmentation, an approach that we have successfully implemented in the conference version of our work\cite{qiaoiros2023}. Moreover, developing a classifier to detect the localizability\cite{nubert2022learning} is also promising.}

\subsection{Real-World Experiments}
\label{sec:real-world}

\begin{table}[]
\centering
\caption{Success Rate on Campus-MS Dataset}
\label{tab:campus}
\resizebox{\columnwidth}{!}{%
\begin{tabular}{cclccccc}
\hline
\multirow{2}{*}{Front-end} & \multirow{2}{*}{Back-end} & \multicolumn{5}{c}{Recall (\%) on Campus-MS $\uparrow$} & \multirow{2}{*}{\begin{tabular}[c]{@{}c@{}}Average\\ Time (ms)\end{tabular}} \\ \cline{3-7}
 &  & \multicolumn{1}{c}{00} & 01 & 02 & 03 & 04 &  \\ \hline
FPFH & TEASER & 89.17 & {\ul 94.59} & 28.57 & 87.16 & 89.68 & 164.88 \\
FPFH & PAGOR & \textbf{95.81} & \textbf{99.61} & 57.14 & \textbf{93.63} & \textbf{96.21} & 190.91 \\
GEM & TEASER & 64.25 & 55.60 & {\ul 71.43} & 70.45 & 65.68 & \textbf{19.54} \\
GEM & PAGOR & {\ul 92.54} & 94.21 & \textbf{92.86} & {\ul 93.03} & {\ul 94.16} & {\ul 54.78} \\ \hline
\end{tabular}%
}
\end{table}

As introduced in Section~\ref{sec:dataset}, we also make a multi-session dataset at a university campus, named Campus-MS, which is collected using a mobile platform equipped with a solid-state LiDAR sensor. This dataset consists of five distinct robot travels (sessions), each collected in different days. To ensure accurate pose ground truth, Real-Time Kinematic (RTK) technology was employed during data collection. For every frame of one travel, we identify frames in other trips if their overlaps exceed 40\%. The quantitative results in Table \ref{tab:campus} show that our method can still keep its registration efficiency and performance without changing any hyperparameters.

In order to verify the proposed G3Reg in practical applications, we embed it into a map merging task, as shown in Figure \ref{fig:campus}. Specifically, we first use LiDAR odometry to build subgraphs, then use GPS to find potential loop closures between multiple trajectories. Considering that GPS could not provide stable transformation, we use G3Reg to obtain relative poses and establish loop edges without pose priors. Finally, a global map is achieved by pairwise consistency maximization (PCM) \cite{mangelson2018pairwise} for false loop removal and global pose graph optimization. The precise map merge result demonstrates the effectiveness of G3Reg in practical applications. 

\section{Conclusion}

In this work, we present G3Reg, an effective and efficient framework for global registration of LiDAR point clouds, which addresses the limitations of conventional approaches. G3Reg introduces two primary contributions: GEMs for constructing correspondences and a distrust-and-verify scheme called PAGOR. By leveraging these innovations, G3Reg significantly improves the robustness and efficiency of global registration. Our extensive comparison experiments demonstrate that G3Reg outperforms SOTA methods. \revised{We present results across diverse hyperparameters and point cloud distributions, showcasing its applicability in various real-world scenarios. Furthermore, the promising directions for future improvement are discussed, including the integration of higher-level descriptors into GEM-based matching or the detection of localizability.}

\bibliographystyle{IEEEtran}
\bibliography{root}

\vfill
\end{document}